\documentclass[journal]{IEEEtran}
\usepackage{amsmath,amsfonts}
\usepackage{algorithmic}
\usepackage{array}
\usepackage[caption=false,font=normalsize,labelfont=sf,textfont=sf]{subfig}
\usepackage{textcomp}
\usepackage{stfloats}
\usepackage{url}
\usepackage{verbatim}
\usepackage{graphicx}
\def\BibTeX{{\rm B\kern-.05em{\sc i\kern-.025em b}\kern-.08em
    T\kern-.1667em\lower.7ex\hbox{E}\kern-.125emX}}
\usepackage{balance}

\makeatletter
\providecommand{\@LN}[2]{}
\makeatother

\usepackage{amsthm}

\usepackage[dvipsnames]{xcolor}
\usepackage{wrapfig}
\usepackage{url} 
\usepackage{tikz}
\usepackage{float}
\usepackage{xcolor}
\usetikzlibrary{fit}

\usetikzlibrary{shapes,arrows}
\usetikzlibrary{positioning}
\usetikzlibrary{decorations.pathreplacing}
\usetikzlibrary{shapes.geometric}
\usepackage{hhline}
\usetikzlibrary{chains, decorations.markings, shadows, shapes.arrows, shapes, fit}
\usepackage{circuitikz}
\usepackage{multirow}
\usepackage{overpic}
\usetikzlibrary{positioning, fit}

\tikzstyle{block} = [draw, rectangle]   
\tikzstyle{input} = [coordinate]
\tikzstyle{output} = [coordinate]

\usepackage{mathtools,leftidx}
\usepackage{mathtools,leftindex,tensor}
\usepackage[version=4]{mhchem} 
\usepackage{enumitem}

\definecolor{White}{rgb}{1,1,1}
\definecolor{WildStrawberry}{rgb}{1.0, 0.24, 0.40}
\definecolor{VioletRed}{rgb}{0.82, 0.13, 0.36}
\definecolor{SpringGreen}{rgb}{0.0, 1.0, 0.5}

\definecolor{mygreen}{RGB}{0, 100, 0}

\newtheorem{thm}{Theorem}

\newtheorem{remark}{\textbf{Remark}}

\usepackage{comment}

\begin{document}

\title{Scene Representation using $\mathbf{360^{\circ}}$ Saliency Graph and its Application in Vision-based Indoor Navigation}

\author{Preeti Meena, ~\IEEEmembership{Graduate Student Member,~IEEE,}
Himanshu Kumar, ~\IEEEmembership{Member,~IEEE,}
Sandeep Yadav, ~\IEEEmembership{Member,~IEEE,} 
\thanks{ Preeti Meena, Himanshu Kumar, and Sandeep Yadav are with Discipline of Electrical Engineering, Indian Institute of Technology Jodhpur, Jodhpur-342030, India {\tt\small meena.52@iitj.ac.in}, {\tt\small hkumar@iitj.ac.in}, {\tt\small sy@iitj.ac.in}}}

\markboth{Journal of \LaTeX\ Class Files,~Vol.~14, No.~8, August~2021}%
{Shell \MakeLowercase{\textit{et al.}}: A Sample Article Using IEEEtran.cls for IEEE Journals}


\maketitle


\begin{abstract}

A Scene, represented visually using different formats such as RGB-D, LiDAR scan, keypoints, rectangular, spherical, multi-views, etc., contains information implicitly embedded relevant to applications such as scene indexing, vision-based navigation. Thus, these representations may not be efficient for such applications. This paper proposes a novel $360^{\circ}$ saliency graph representation of the scenes. This rich representation explicitly encodes the relevant visual, contextual, semantic, and geometric information of the scene as nodes, edges, edge weights, and angular position in the $360^{\circ}$ graph. Also, this representation is robust against scene view change and addresses challenges of indoor environments such as varied illumination, occlusions, and shadows as in the case of existing traditional methods. We have utilized this rich and efficient representation for vision-based navigation and compared it with existing navigation methods using $360^{\circ}$ scenes. However, these existing methods suffer from limitations of poor scene representation, lacking scene-specific information. This work utilizes the proposed representation first to localize the query scene in the given topological map, and then facilitate $2$D navigation by estimating the next required movement directions towards the target destination in the topological map by using the embedded geometric information in the $360^{\circ}$ saliency graph. Experimental results demonstrate the efficacy of the proposed $360^{\circ}$ saliency graph representation in enhancing both scene localization and vision-based indoor navigation.

\end{abstract} 

\section{Introduction}
\label{sec:intro}
\IEEEPARstart{I}{nterpreting} and representing visual scenes in a compact, meaningful, and computationally efficient manner is a key requirement for a wide range of vision-based applications, including scene retrieval, localization, and navigation \cite{meena2024indoor}. Indoor environment comprises scenes with varied illumination, complex layouts, occlusions, high visual similarity between different locations, and shadows \cite{meena2024indoor}. These challenges make indoor scene representation a widely explored and still open research problem. Earlier methods have relied on RGB images, RGB-D data, LiDAR point clouds, keypoint-based descriptors, and multi-view images as scene representations \cite{qin2021semantic}. RGB images are generally encoded through visual features extracted from them, but such features are highly sensitive to varying indoor conditions, like varied illumination and presence of shadows. LiDAR point clouds effectively capture the geometric structure of the scene, but tend to be sparse and lack semantic understanding. Multi-view images are introduced to address the limited field of view or partial scene information captured by a single image \cite{meena2024indoor}. However, multi-view-based representation requires additional camera parameters to infer spatial relationships, which increases computational complexity. These limitations highlight that scene representations relying on visual and structural layout information are prone to errors and can lead to inaccurate results in vision-based applications such as indoor navigation. To mitigate these issues, several methods \cite{labrosse2007short, yu2011image} have employed $360^{\circ}$ visual data via panorama images, which provide a complete and continuous view of scene by capturing richer scene context from a single observation point.
The queryable semantic topological maps \cite{mehan2024questmaps} comprise key locations, and their connections can facilitate navigation within indoor environments. Existing navigation methods that use topological maps \cite{an2024etpnav} rely on panorama images to represent the scene at each key location. This results in suboptimal scene representation because it treats all components within the image as equally important, neglecting the uniqueness of different components in navigation tasks. Instead of using the entire panorama image, several methods \cite{kwon2024wayil, raj2016appearance, chen2019behavioral, chaplot2020neural} employ visual features extracted from the panorama image as a representation for navigation. However, visual features \cite{raj2016appearance, chen2019behavioral, chaplot2020neural} are unreliable in indoor environments due to varying illumination and shadows. 
Thus, a concise and navigation-specific scene representation of the panorama image content, such as a `summary' \cite{zhang2014panocontext, niu2019hand, du2020learning} is required.


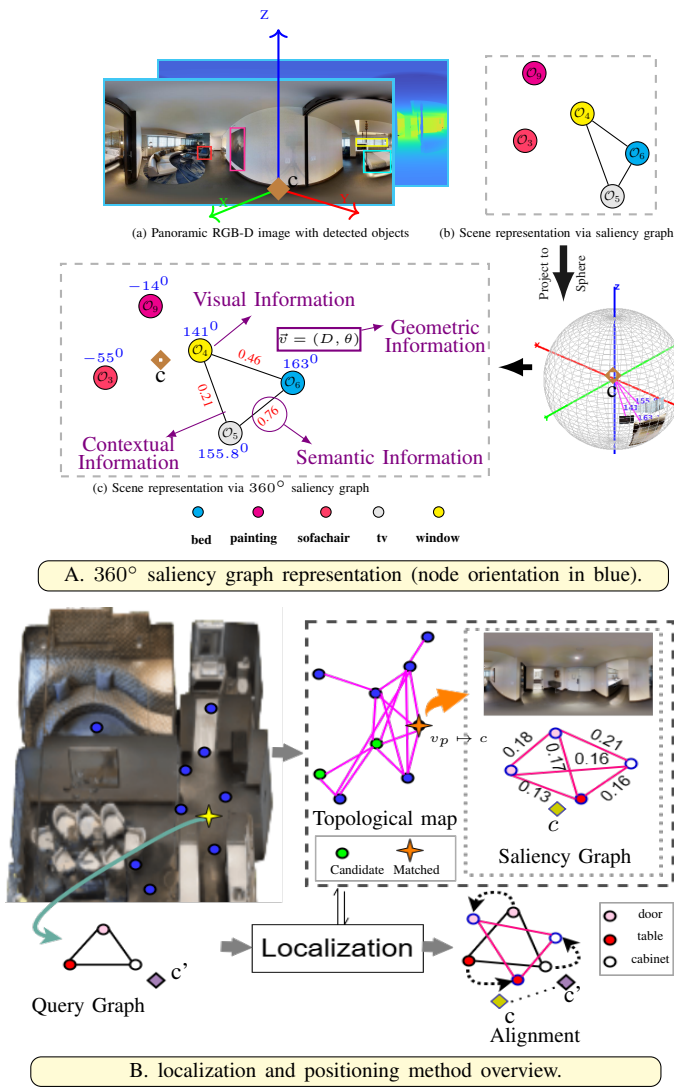
\begin{figure}[!t]
\setlength{\fboxsep}{0.9pt}     
\tikzstyle{block} = [draw, rectangle, dashed, draw=black!40, thick, inner sep=0pt, outer sep=0pt]  
\tikzstyle{dash_blk} = [draw, rectangle, dashed, draw=black!40, thick, inner sep=0pt, outer sep=0pt] 
\tikzstyle{block1} = [draw, thick, minimum height=3em, minimum width=6em]
\tikzstyle{input} = [coordinate]
\tikzstyle{output} = [coordinate]
\tikzstyle{dotted_block}=[draw, rectangle, line width=1pt, dash pattern=on 1.5pt off 3pt on 6.5pt off 3pt, rounded corners]
\tikzstyle{obnode}=[draw,circle,inner sep=1.4pt]
\tikzstyle{obbnode}=[draw,circle,inner sep=0.5pt]
\tikzstyle{dinode}=[diamond,draw,inner sep=1.2pt]

\centering
\begin{tikzpicture}

\node [input, name=input] {};  

\node [above right=-1.25cm and -3.9cm  of input] (ip1) {\colorbox{cyan!60}{\includegraphics[width=.21\textwidth, height=.09\textwidth]{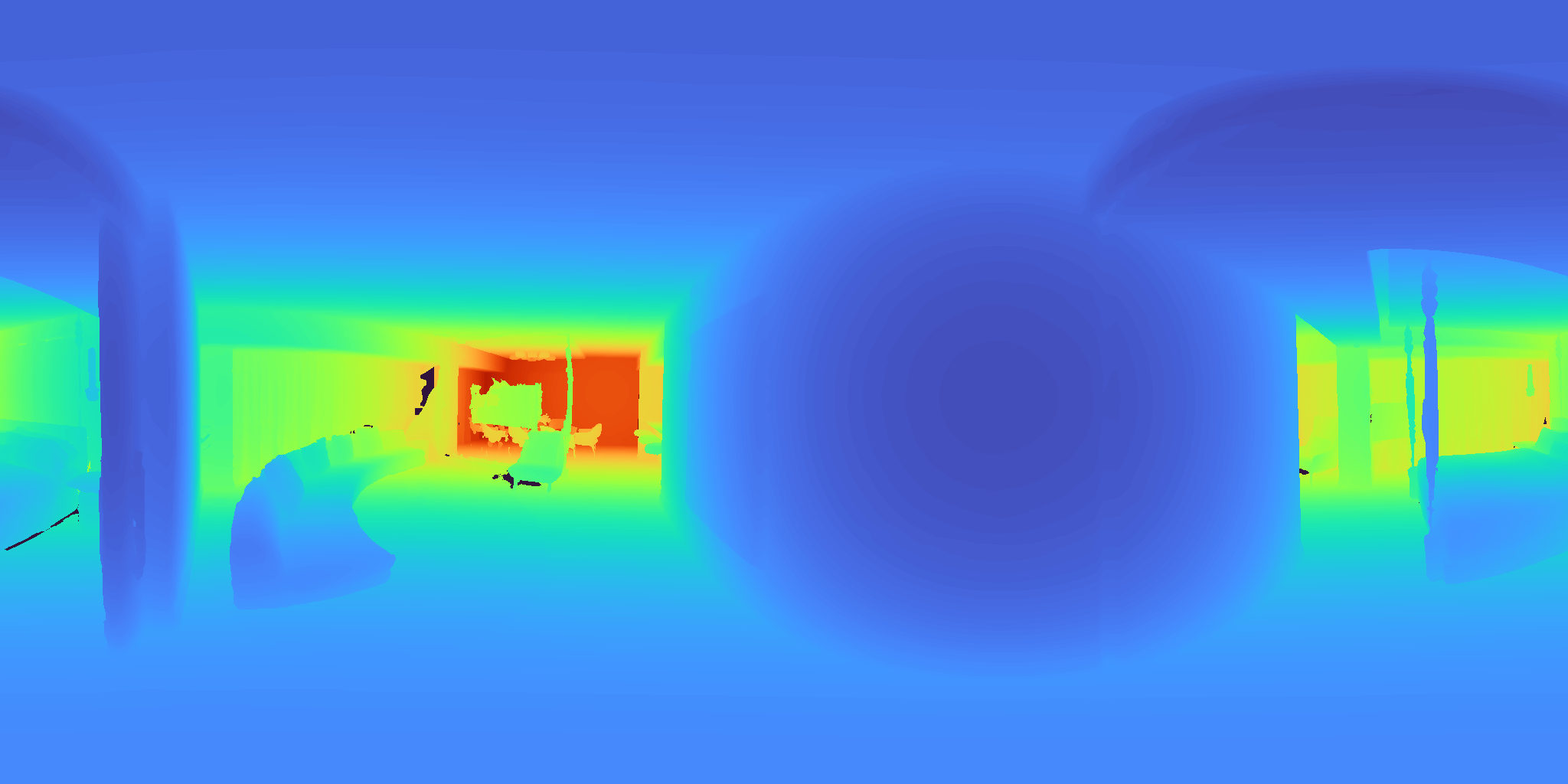}}};
\node [above left = -2.2cm and -3.4cm of ip1] (ip2) {\colorbox{cyan!60}{\includegraphics[width=.21\textwidth, height=.09\textwidth]{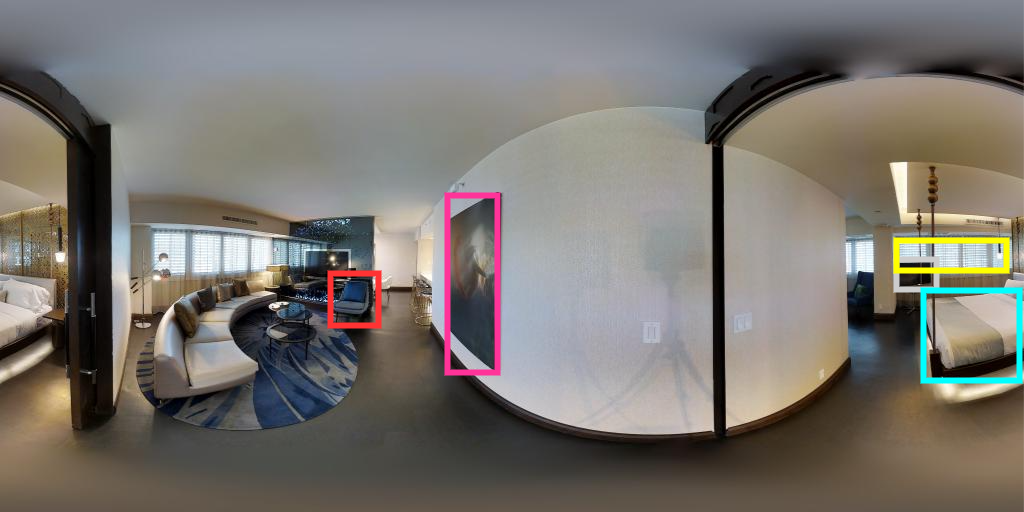}}};

\begin{scope}[shift={(ip1.center)}, scale=2.5] 
    \draw[->, thick, green] (-0.13, -0.35) -- (-0.5, -0.5) node[anchor=south west] {\tiny{X}};
    \draw[->, thick, blue] (-0.13, -0.35) -- (-0.13, 0.5) node[anchor=south east] {\tiny{Z}}; 
    \draw[->, thick, red] (-0.13, -0.35) -- (0.3, -0.48) node[anchor=south east] {\tiny{Y}};
\draw[thick, brown, line width=1.5mm] (-0.13-0.02, -0.35) -- (-0.13, -0.35-0.02) -- (-0.13+0.02, -0.35) -- (-0.13, -0.35+0.02) -- cycle;
\node at (-0.13, -0.31) [anchor=west] {\small{c}}; 
\end{scope}
\node [below left= 0.3cm and -3.6cm of ip1, align=center] (ipl1) {\tiny{(a) Panoramic RGB-D image with detected objects}};

\node [right of= ipl1,  node distance=3.8cm, align=center] (ipl2) {\tiny{(b) Scene representation via saliency graph}};

\node [below left= 2.95cm and -3.4cm of ipl1, align=center] (ipl3) {\tiny{(c) Scene representation via $360^{\circ}$ saliency graph}};

\node [below left = -1.6cm and -4.8cm of ip1, align=center] (pp1a) {};
\node [below left = -1.6cm and -4.1cm of ip1, align=center] (pp1) {}; 
\node [obbnode,fill= Magenta, below left = -0.4cm and -1.4cm of pp1, align=center] (pp9) {\scalebox{0.5}{$\mathcal{O}_{9}$}}; 
\node [obbnode,fill=WildStrawberry, below right =0.5cm and 0.1cm of pp1a, align=center] (pp3) {\scalebox{0.5}{$\mathcal{O}_{3}$}}; 
\node [below left = 0.5cm and 0.2cm of pp3, align=center] (pp2) {}; 
\node [above of= pp2, node distance=1.1cm, align=center] (pp8) {};

\node [obbnode,fill=Yellow, right of=pp8, node distance=1.2cm, align=center] (pp4) {\scalebox{0.5}{$\mathcal{O}_{4}$}}; 
\node [obbnode,fill= black!10, right of= pp2, node distance=1.6cm, align=center] (pp5) {\scalebox{0.5}{$\mathcal{O}_{5}$}}; 
\node [obbnode,fill=cyan, below right = 0.3cm and 0.5cm of pp4, align=center] (pp6) {\scalebox{0.5}{$\mathcal{O}_{6}$}}; 

\draw [draw,-] (pp4) -- node[sloped, below, color=red, yshift=0.5mm] {} (pp5);
\draw [draw,-] (pp4) -- node[sloped, above, color=red, yshift=-0.5mm] {} (pp6);
\draw [draw,-] (pp5) -- node[sloped, below, color=red, yshift=0.5mm] {} (pp6);

\node [dash_blk, draw=black!30, inner sep=2.0pt, xshift=0.4pt, yshift=-0.2pt, fit = (pp1a) (pp6) (pp4) (pp5) (pp9)] (gsra) {};


\node [below right = 1.0cm and 0.9cm of ip1, align=center] (sp) {\includegraphics[width=0.12\textwidth]{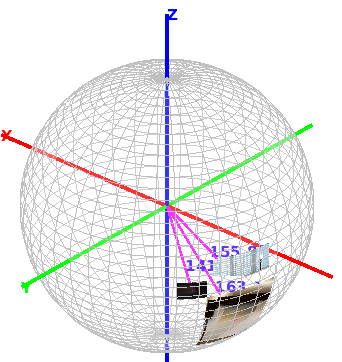}};

 \node [below right = 0.28cm and 0.7cm of pp2, align=center] (arr) {};
 \node [below of= arr,  node distance=1.0cm, align=center] (arr1) {};

\draw [draw, ->, line width=2.8pt, >=latex, color=black!90] (arr) -- 
    node[above, align=center, rotate=90] {\tiny{Project to}} 
    node[below, align=center, rotate=90] {\tiny{Sphere}} 
    (arr1);

\begin{scope}[shift={(sp.center))}, scale=1] 
   \draw[thick, brown, line width=1.5mm] (-0.16+0.1, -0.11) -- (-0.15+0.1, -0.12) -- (-0.14+0.1, -0.11) -- (-0.15+0.1, -0.10) -- cycle; 

    \node at (-0.06, -0.13) [anchor=north] {\small{c}}; 
\end{scope}

 \node [left of = sp, node distance=1.0cm, align=center] (srr) {};
 \node [left of= srr,  node distance=0.7cm, align=center] (srr1) {};
\draw [draw, ->, line width=2.2pt, >=latex, color=black] (srr) --  node[above, align=center] {} (srr1);
\node [below left = -2.0cm and 5.9cm of sp, align=center] (p1a) {};
\node [below right = 1.5cm and -4.1cm of ip1, align=center] (p1) {}; 
\node [obbnode,fill= Magenta, below left = -0.4cm and -0.13cm of p1, align=center] (p9) {\scalebox{0.5}{$\mathcal{O}_{9}$}}; 
\node [obbnode,fill=WildStrawberry, below right =0.5cm and 0.2cm of p1a, align=center] (p3) {\scalebox{0.5}{$\mathcal{O}_{3}$}}; 
\node [below left = 0.5cm and 0.7cm of p3, align=center] (p2) {}; 
\node [above of= p2, node distance=1.1cm, align=center] (p8) {};

\node [obbnode,fill=Yellow, right of=p8, node distance=2.2cm, align=center] (p4) {\scalebox{0.5}{$\mathcal{O}_{4}$}}; 
\node [obbnode,fill= black!10, right of= p2, node distance=2.6cm, align=center] (p5) {\scalebox{0.5}{$\mathcal{O}_{5}$}}; 
\node [obbnode,fill=cyan, below right = 0.2cm and 1.0cm of p4, align=center] (p6) {\scalebox{0.5}{$\mathcal{O}_{6}$}}; 

\draw [draw,-] (p4) -- node[sloped, below, color=red, yshift=0.5mm] {\tiny{0.21}} (p5);
\draw [draw,-] (p4) -- node[sloped, above, color=red, yshift=-0.5mm] {\tiny{0.46}} (p6);
\draw [draw,-] (p5) -- node[sloped, below, color=red, yshift=0.5mm] {\tiny{0.76}} (p6);

\begin{scope}[shift={(-2.1, 0)}, scale=1] 
    \draw[thick, brown, line width=1.5mm] (-1.6, -3.43) -- (-1.63, -3.46) -- (-1.66, -3.43) -- (-1.63, -3.4) -- cycle; 
    \node at (-1.63, -3.46) [anchor=north] {\small{c}}; 
\end{scope}

\node[above left=-0.06cm and -0.5cm of p1.north west, font=\bfseries,blue] (so1){};
\node[above left=-0.06cm and -0.5cm of p3.north west, font=\bfseries,blue] (so1){\textbf{\tiny{$-55^0$}}};
\node[above left=-0.04cm and -0.5cm of p4.north west, font=\bfseries,blue] (so1){\textbf{\tiny{$141^0$}}};
\node[below left=0.13cm and -0.5cm of p5.north west, font=\bfseries,blue] (so1){\textbf{\tiny{$155.8^0$}}};
\node[above left=-0.04cm and -0.6cm of p6.north west, font=\bfseries,blue] (so1){\textbf{\tiny{$163^0$}}};
\node[above left=-0.06cm and -0.5cm of p9.north west, font=\bfseries,blue] (so1){\textbf{\tiny{$-14^0$}}};


\node[below left=0.03cm and -0.3cm of p4.north west, font=\bfseries,blue] (via){};
\node [above right= 0.36cm and 0.5cm of via, align=center] (via1) {};

\draw [draw, ->, >=latex, color=violet] (via) --  node[above, align=center] {} (via1);
\node[above left=-0.1cm and -1.5cm of via1.west,text=violet] (vi) {\footnotesize{Visual Information}};

\node[draw=violet, circle, inner sep=4.6pt, above right=-0.03cm and 0.47cm of p5.north west] (sia) {};
\node [below right= 0.2cm and 0.4cm of sia, align=center] (sia1) {};
\draw [draw, ->, >=latex, color=violet] (sia) --  node[above, align=center] {} (sia1);
\node[above left=-0.3cm and -2.4cm of sia1.west, text=violet] (si) {\footnotesize{Semantic Information}};

\node[above left=0.1cm and -0.07 cm of p5.north east] (ccia) {};
\node [below left= 0.2cm and 0.8cm of ccia, align=center] (ccia1) {};
\draw [draw, ->, >=latex, color=violet] (ccia) --  node[above, align=center] {} (ccia1);
\node[above left=-0.25cm and -0.5cm of ccia1.west, text=violet] (cci) {\footnotesize{Contextual}};
\node[below of=cci,node distance=0.24cm, align=center, text=violet] (cci1) {\footnotesize{Information}};
\node[draw=violet, block1, minimum height=0.3cm, minimum width=1.1cm,inner sep=0.0pt, above right=0.3cm and -0.1cm of p6.north west] (gsia) {\tiny{$\vec{v}=(D,\theta)$}};
\node [below right= -0.5cm and 0.3cm of gsia, align=center] (gsia1) {};
\draw [draw, ->, >=latex, color=violet] (gsia) --  node[above, align=center] {} (gsia1);
\node[below left=-0.17cm and -1.4cm of gsia1.west, text=violet] (gsi) {\footnotesize{Geometric}};
\node[below of=gsi,node distance=0.24cm, align=center, text=violet] (gsi1) {\footnotesize{Information}};
\node [dash_blk, draw=black!30, inner sep=0.3pt, xshift=-2.7pt, yshift=1.0pt, fit = (p1) (p6) (p3) (p4) (p5) (p9) (so1) (cci1) (gsi) (gsi1)] (gsr) {};

\node [obnode,fill=cyan, below left =2.4cm and -0.7cm of p1, align=center] (o1) {};
\node[above left=-0.6cm and -0.36cm of o1.north west] (l2) {\textbf{\tiny{bed}}};
\node [obnode,fill=Magenta, right of=o1, node distance=0.8cm, align=center] (o2) {};
\node[above left=-0.61cm and -0.43cm of o2.north west] (l3) {\textbf{\tiny{painting}}};
\node [obnode,fill=WildStrawberry, right of=o2, node distance=0.9cm, align=center] (o3) {};
\node[above left=-0.58cm and -0.5cm of o3.north west] (l4) {\textbf{\tiny{sofachair}}};
\node [obnode,fill=black!10, right of=o3, node distance=0.7cm, align=center] (o4) {};
\node[above left=-0.58cm and -0.3cm of o4.north west] (l4) {\textbf{\tiny{tv}}};
\node [obnode,fill=Yellow, right of=o4, node distance=0.8cm, align=center] (o6) {};
\node[above left=-0.58cm and -0.46cm of o6.north west] (l4) {\textbf{\tiny{window}}};

\node[below left=1.3cm and -1.9cm of sp, align=center, draw, text width=0.93\columnwidth, node distance=1cm, rounded corners, fill=yellow!20, inner sep=2pt, outer sep=0pt] (cp) {\footnotesize{A. $360^{\circ}$ saliency graph representation (node orientation in blue).}};

\node [below left = 1.8cm and -2.1cm of sp, align=center] (ov) {\includegraphics[width=0.49\textwidth, height=2.1in]{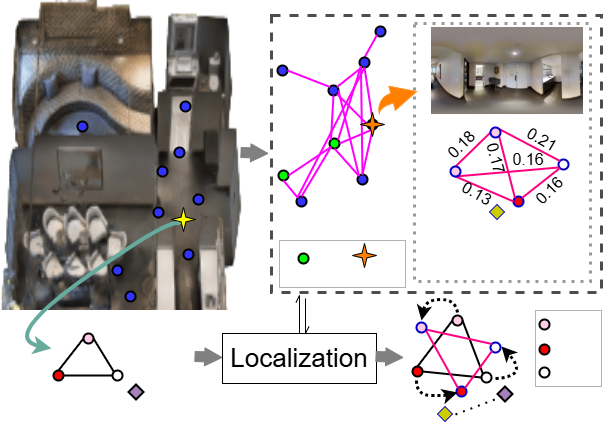}};

\node[above right=-2.1cm and -3.5cm of ov, align=center] (vpnode) {\tiny{$v_p\mapsto c$}};

\node[below right=0.74cm and 0.55cm of vpnode, align=center] (vpnode1) {\small{$c$}};

\node[above right=-3.2cm and -5.05cm of ov, align=center] (tpg) {\footnotesize{Topological map}};

\node[below of=tpg, node distance=0.66cm, align=center] (can) {\tiny{Candidate \ Matched}};

\node[below right=-0.0cm and 0.3cm of tpg, align=center] (cag) {\footnotesize{Saliency Graph}};

\node[below left=-0.4cm and -2.1cm of ov, align=center] (qg) {\footnotesize{Query Graph}};
\node[above right=-0.01cm and 0.1cm of qg, align=center] (refa2) {\small{c'}};

\node[below right=-0.2cm and -2.5cm of ov, align=center] (ref1) {\small{c}};
\node[above right=-0.12cm and 0.4cm of ref1, align=center] (ref2) {\small{c'}};

\node[below right=0.01cm and -2.7cm of ov, align=center] (rea) {\footnotesize{Alignment}};

\node[above right=0.06cm and 0.4cm of ref2, align=center] (obj3) {\tiny{cabinet}};
\node[above of=obj3, node distance=0.3cm, align=center] (obj2) {\tiny{table}};
\node[above of=obj2, node distance=0.3cm, align=center] (obj1) {\tiny{door}};

\node[below left=7.9cm and -1.8cm of sp, align=center, draw, text width=0.93\columnwidth, node distance=1cm, rounded corners, fill=yellow!20, inner sep=2pt, outer sep=0pt] (cp1) {\footnotesize{B. localization and positioning method overview.}};

\end{tikzpicture}
\vspace{-0.6cm}
\caption{(A) Orientation-rich $360^{\circ}$ saliency graph representation of a scene, constructed from an RGB-D panorama image. (B) Localization in a topological map and positioning within a localized scene using $360^{\circ}$ saliency graphs. Localization: via scene graph matching using visual, semantic, and contextual information embedded in the $360^0$ saliency graph; Positioning in the scene: by aligning geometric information in a $360^0$ saliency graph. }
\label{fig:360graph}
\vspace{-4.5mm}
\end{figure}


A summary of a panorama image can be generated by detecting objects in the image \cite{guerrero2020s, gao2021room, mehan2024questmaps}, and subsequently can be utilized to represent scenes at key locations of topological graphs \cite{wang2024graph}. 
However, listing all the objects in the scene is not an effective scene representation, as the scene at various locations often contains similar objects in similar environments, e.g., office rooms, meeting spaces, and conference rooms are typically composed of identical objects (e.g., chairs, tables, etc.). Consequently, object-based scene representation is prone to failure in localizing such scenarios. This limitation motivates the use of the salient objects, which offer a more distinctive and discriminative scene representation, for scene summary. Saliency highlights prominent objects by assigning importance scores based on their distinctiveness, reducing the impact of irrelevant objects \cite{meena2024volumetric}. For example, an oven in a kitchen or a bathtub in a bathroom would be deemed more informative than common items such as chairs. However, representing a scene solely as a collection of salient objects still lacks structural and contextual information, which are crucial for robust and accurate positioning in navigation tasks. 
Recent methods \cite{du2020learning} leverage contextual information along with objects to improve efficiency. However, these methods consider all objects and their relations equally important, which can lead to suboptimal performance, as each component has a different significance in the unique scene representation. Several existing methods \cite{qin2021semantic} employ semantic graphs for scene representation and positioning in navigation tasks. However, the effectiveness of semantic graphs is highly dependent on the accuracy of object detection and classification. To address this limitation, the proposed approach focuses on salient objects extracted using volumetric saliency, which emphasizes larger objects within the scene. This helps to mitigate the impact of undetected smaller or occluded objects. Also, the semantic graph representation-based topological map is not well-suited for orientation estimation because it does not preserve geometric information. Thus, path planning cannot be performed.

This paper introduces a navigation-specific scene representation through $360^{\circ}$ saliency graph generated from an RGB-D panorama image. For navigation, a topological map of an indoor environment in which each key location is accompanied by a $360^{\circ}$ saliency graph of that location is constructed. A $360^{\circ}$ saliency graph effectively encodes visual, contextual, semantic, and geometric information of scene by representing salient objects as nodes, relations between object pairs as edges, object's label and edge weights as semantic, and descriptors like angular position, saliency score, $3$D location as geometric information. Salient objects are identified using volumetric saliency \cite{meena2024volumetric}, which is based on the observation that larger objects occupy a greater portion of a scene and thus contribute more significantly to its representation \cite{oliva2020ganalyze}. A key advantage of volumetric saliency \cite{meena2024volumetric} is its ability to detect scene-defining objects regardless of their position in the foreground or background. The representation of a $360^{\circ}$ saliency graph on a sphere enables us to detect orientation from the relative positions of salient objects. As shown in Figure \ref{fig:360graph}, given a RGB-D panorama image of an unknown location, we utilize the $360^{\circ}$ saliency graph as scene representation and compare it against those in the pre-constructed queryable topological map for the candidate graph or location selection. Then, a best-matched candidate graph is utilized to estimate the pose accurately by aligning the reference points and salient nodes. A panoramic image can be explicitly represented via a $360^{\circ}$ saliency graph compared to a saliency graph, as shown in Fig. \ref{fig:360graph}. 
The key contributions of the paper are as follows: 
          (i) A navigation-specific scene representation method through an orientation-rich $360^{\circ}$ saliency graph derived from a panorama RGB-D image. The proposed graph jointly encodes \textit{Visual}, \textit{Semantic}, \textit{Contextual}, and \textit{Geometric} information within a unified spherical structure that covers the entire $360^{\circ}$ field of view, as illustrated in Fig. \ref{fig:360graph}. Unlike existing navigation-oriented representations (which either rely on a single information or treat all components equally and different types of information independently), the proposed graph integrates heterogeneous scene's information into a structurally consistent spherical graph. This yields a compact yet highly discriminative scene representation that preserves both spatial layout and local orientation relationships characteristic of indoor scenes. This new representation facilitates efficient and robust localization for navigation, demonstrated in the experimental results of the paper. (ii) A $360^{\circ}$ saliency graph representation-based localization and positioning approach to facilitate indoor navigation. This work analyzes the performance of the proposed work for $2$D navigation on a given topological map by matching a query graph to the reference $360^{\circ}$ saliency graphs. The relative orientation of salient objects in the query graph with respect to those encoded in the spherical $360^{\circ}$ saliency graph is exploited for positioning and navigation planning.         
        

The remainder of the paper is organized as follows: Section \ref{sec:prior} discusses related works on scene representation used for indoor navigation. Section \ref{sec:method} presents the proposed $360^{\circ}$ saliency graph-based representation, along with the localization and positioning methods for navigation. Experimental results and comparisons are presented in Section \ref{sec:results}. The paper concludes in Section \ref{sec:conclusion} with a discussion on future works.

\vspace{-0.1cm}
\section{Prior Works}
\label{sec:prior}
\vspace{-0.13cm}

Effective scene representation is an integral part of vision-based applications such as indoor navigation \cite{koh2021pathdreamer, zhou2022relational, mehan2024questmaps}.
Methods designed for navigation in a known environment \cite{moravec1985high} utilize available maps, while navigation in an unknown environment \cite{konolige2011navigation, xu2024robot} requires mapping, i.e., map creation.   

Visual navigation in known environments, as per the focus of this paper, utilizes a variety of navigational-specific scene representations \cite{labrosse2007short, yu2011image, an2024etpnav, liu2024volumetric, taioli2024mind} for effective localization, positioning and orientation detection.  
Labrosse et al. \cite{labrosse2007short} utilized whole $2$D panorama images rather than a few landmarks extracted from images. Yu et al. \cite{yu2011image} utilized a sequence of snapshot images for landmark-based navigation. 
Cha et al. \cite{cha2012omni} use optic flow information by employing the Kanade-Lucas-Tomashi(KLT) algorithm to estimate the movement of agents for homing navigation. Lee et al. \cite{lee2018visual} utilized Haar-like Features for scene representation. 
Recently, Koh et al. \cite{koh2021pathdreamer} proposed a stochastic hierarchical visual world model termed Pathdreamer that represents the scene via visual, spatial, and semantic information to generate realistic $360^{\circ}$ panorama images for unseen trajectories in real-world buildings. 
Krantz et al. \cite{krantz2021waypoint} utilize feature-based scene representation in which the RGB-D panorama image is encoded with a ResNet-$50$ \cite{he2016deep}. Several proposed works \cite{du2021curious, khandelwal2022simple, li2022envedit, hong2023learning} learn the scene representations with a contrastive learning approach. Hong et al. \cite{hong2023learning} employed a visual representation learning method (Ego$^2$-Map) that trains a visual‑transformer by contrasting egocentric RGB views with semantic maps. However, Ego$^2$-Map requires pre‑constructed semantic maps (or dense semantic annotations) for training, which imposes a significant data-collection cost and limits applicability to only annotated or well‑mapped environments.
Pramatarov et al. \cite{pramatarov2022boxgraph} introduced BoxGraph, extending the camera-based semantic graph approach of \cite{qin2021semantic} to $3$D LiDAR point clouds representation by replacing $2$D bounding boxes with $3$D semantic blobs. Their method estimates $6$-DoF pose using SVD on vertex centroids. However, it relies on accurate alignment between query and reference points, making it sensitive to noise and lacking built-in outlier rejection, requiring additional preprocessing such as RANSAC. Moreover, the use of $3$D LiDAR point clouds captures partial objects and lacks explicit orientation information. Xu et al. \cite{xu2023free} proposed a panoramic image-based rendering algorithm for free-viewpoint navigation of indoor scenes. The authors represent each input RGB-D panorama with a set of spherical superpixels and warp each superpixel individually. 
Kim et al. \cite{kim2023topological} utilized a topological semantic graph memory (TSGM) for navigation, which has two types of nodes: image nodes and object nodes, and a cross-graph mixer module that takes the collected nodes to get contextual information.
More recently, Wang et al. \cite{wang2024graph} proposed a cross-modal attention navigation framework and an environment representation graph through object detection and GCN.
An et al. \cite{an2024etpnav} proposed a method, ETPNav, that performs online topological mapping of environments for navigation by self-organizing predicted waypoints along a traversed path without prior environmental experience. The nodes/key locations of the topological map are represented through the feature vectors extracted from panorama RGB-D images using visual encoders.
Liu et al. \cite{liu2024volumetric} proposed a volumetric environment representation (VER), which aggregates the perspective features into structured $3$D cells and utilizes volume state estimation for navigation.
Liu et al. \cite{liu2024multiple} leverage multiple visual features (e.g., RGB feature, texture, color, object type, depth feature) for scene representation of nodes within the topological map. However, the method does not explicitly model geometric or contextual relations between different regions within the scene. Moreover, the final visual feature used for scene representation is obtained by summing multiple visual‑feature streams into a single descriptor, which may blur the fine structural or orientation‑specific information (e.g., object orientation or relative location).

Existing visual navigation methods represent the scenes either using a panorama image's visual feature or assigning equal importance to all components/objects present within a scene, resulting in poor scene representation. Moreover, localization based on visual and object semantic information leads to suboptimal performance. This paper introduces a $360^{\circ}$ saliency graph representation that effectively encodes visual, contextual, object semantic, and edge semantic information. Alongwith, this graph incorporate salient-object-based orientation information for navigation controls. 

In addition to visual navigation, several approaches leverage natural language instructions within the visual environment, giving rise to Vision-and-Language Navigation (VLN) methods. VLN involves not only interpreting visual inputs but also aligning them with linguistic cues for efficient navigation. Recent VLN approaches, including \cite{wang2019reinforced}, \cite{qi2020reverie}, \cite{gao2021room}, \cite{guhur2021airbert}, \cite{qiao2022hop}, \cite{chen2022think}, and \cite{qiao2023march}, utilize either visual features from panoramic images or object-level representations as scene representations for localization within the navigation pipeline. This often results in suboptimal performance due to limited encoding of scene context. In contrast, this paper utilizes salient objects along with their semantic and geometric properties, orientation, and the structural layout of the scene captured through contextual relationships between salient objects. 

\setlength{\fboxsep}{0.9pt}     
\tikzstyle{block} = [draw, rectangle]  
\tikzstyle{dash_blk} = [draw, rectangle, dashed, draw=black!40, thick, inner sep=0pt, outer sep=0pt]  
\tikzstyle{block1} = [draw=none, thick, minimum height=3em, minimum width=6em]
\tikzstyle{input} = [coordinate]
\tikzstyle{output} = [coordinate]
\tikzstyle{dotted_block}=[draw, rectangle, line width=1pt, dash pattern=on 1.5pt off 3pt on 6.5pt off 3pt, rounded corners]
\tikzstyle{obnode}=[draw,circle,inner sep=1.4pt]
\tikzstyle{dinode}=[diamond,draw,inner sep=1.2pt]

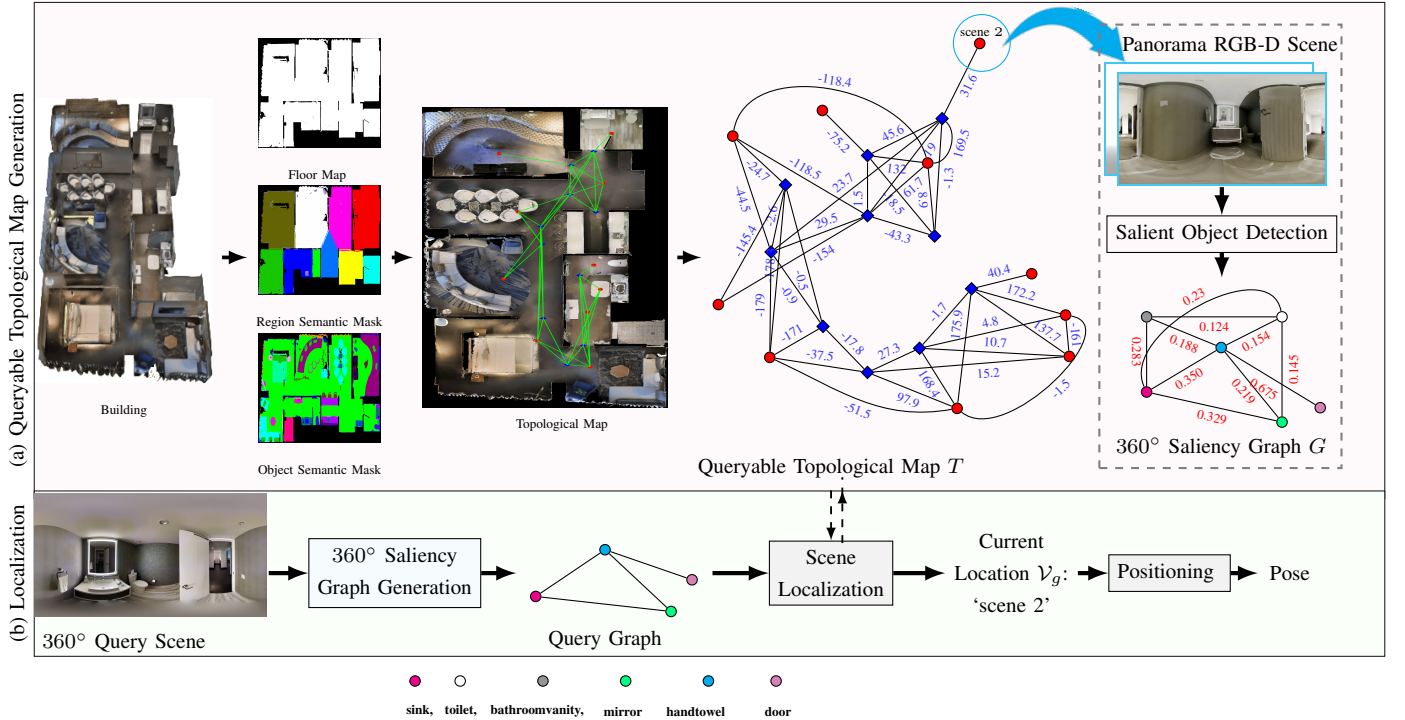
\begin{figure*}[!hbt]
\centering
\begin{tikzpicture}
 
\node [input, name=input] {};

\node [block, draw=black!90, fill=purple!2, below right=-3.4cm and -8.3cm of input, minimum height=6.6cm, minimum width=0.99\textwidth] (TM) {};

\node [block, draw=black!90, fill=green!2, below right=3.1cm and -8.3cm of input, minimum height=2.2cm, minimum width=0.99\textwidth] (loca) {};

\node [left of=input,  node distance=0.4cm, align=center] (A) {};

\node [left of=input,  node distance=7.1cm, align=center] (i1) {\includegraphics[width=.13\textwidth]{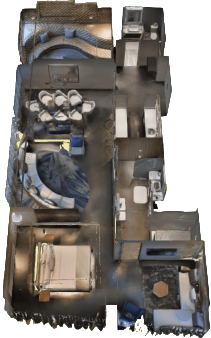}\\ \tiny{Building}};

\node [right of=i1,  node distance=2.6cm, align=center] (i2) {\includegraphics[width=.09\textwidth, height=.08\textwidth]{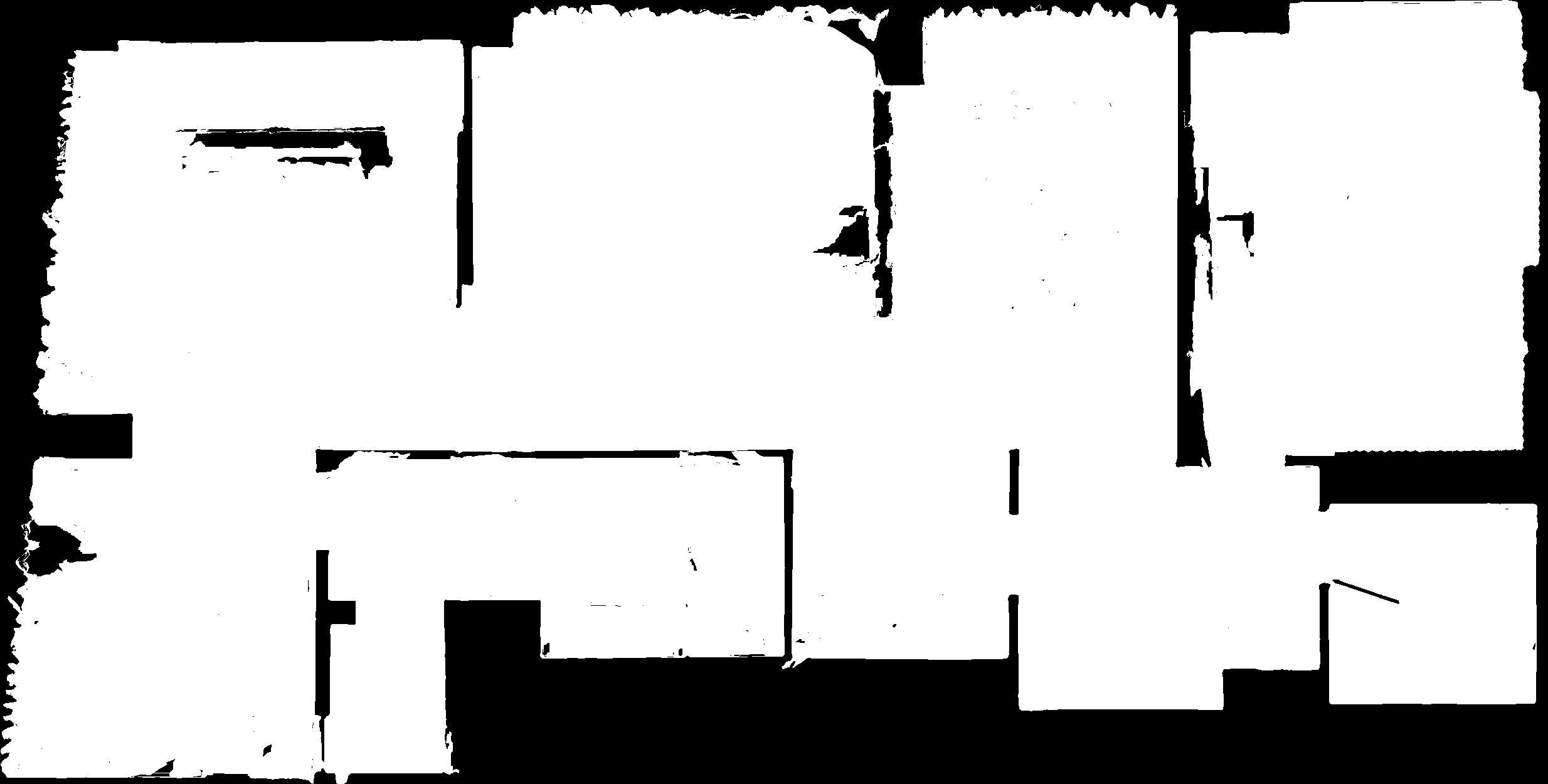}\\ \tiny{Floor Map} \vspace{0.05cm}\\ \includegraphics[width=.09\textwidth, height=.08\textwidth]{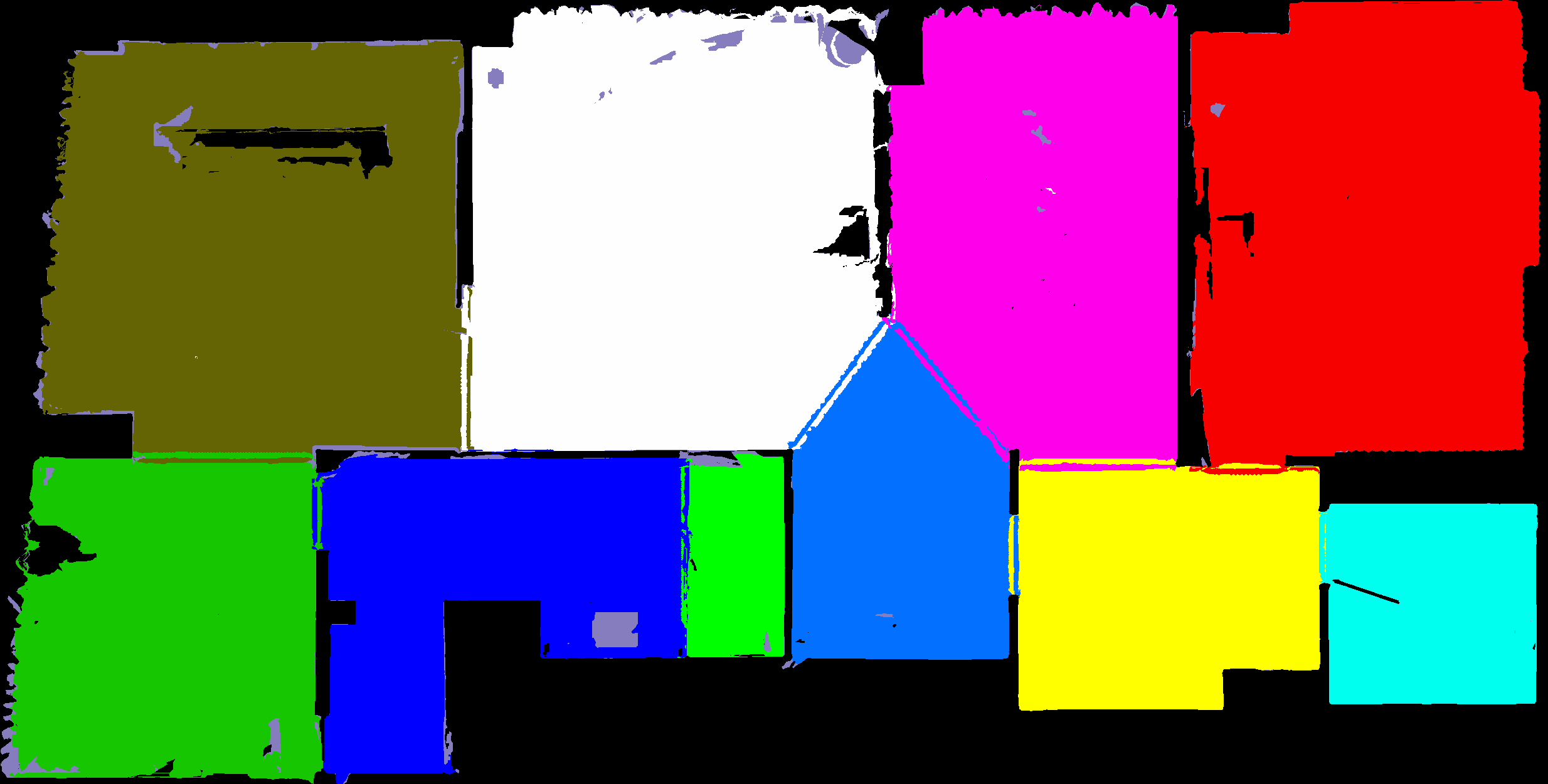}\\ \tiny{Region Semantic Mask}\vspace{0.05cm} \\ \includegraphics[width=.09\textwidth, height=.08\textwidth]{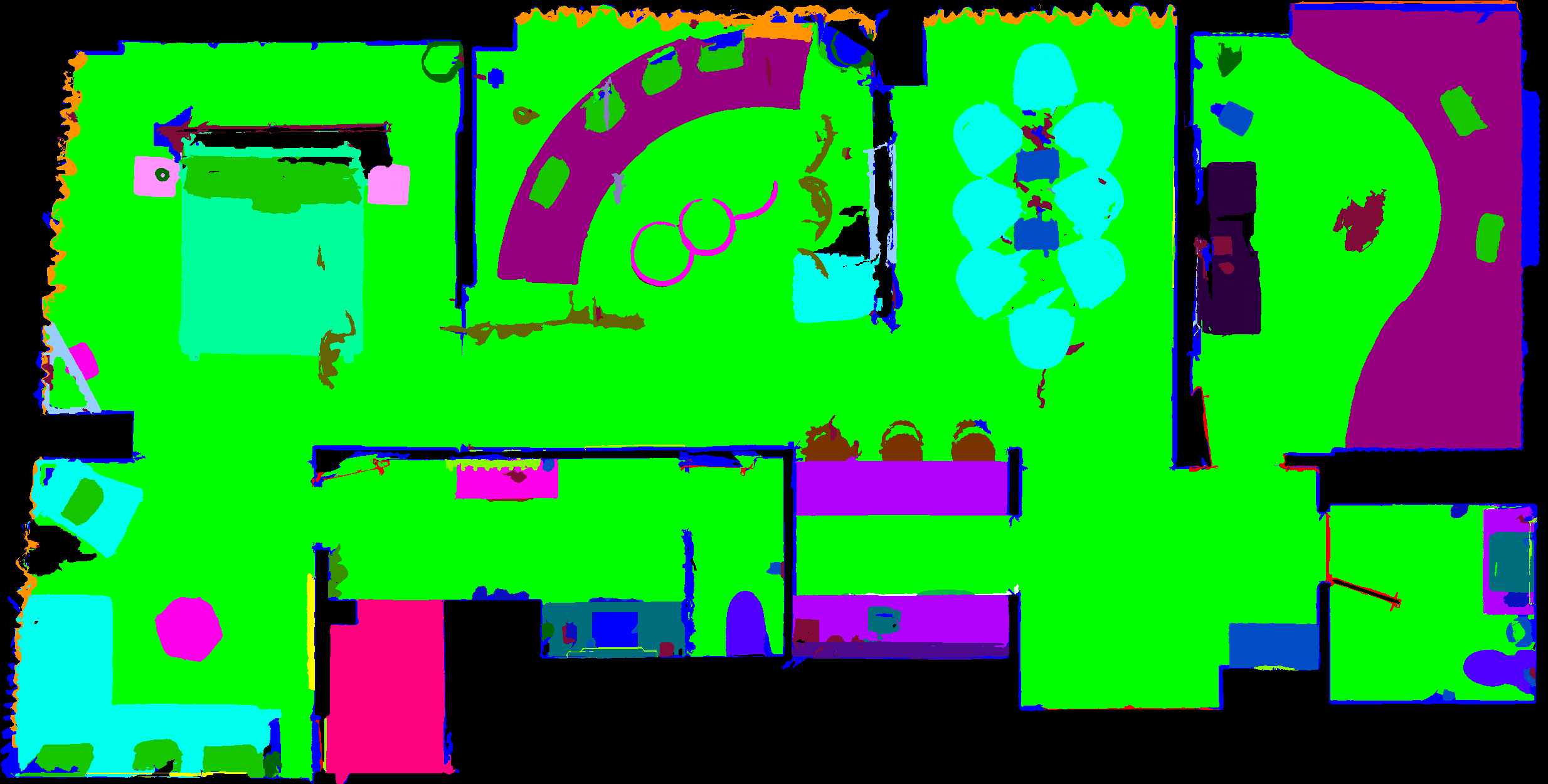}\\ \tiny{Object Semantic Mask}  };

\draw [->,line width=1.4pt, >=latex] (i1) -- node {} (i2);

\node [right of=i2,  node distance=3.0cm, align=center, rotate=90] (i3) {\includegraphics[width=.22\textwidth, height=.18\textwidth]{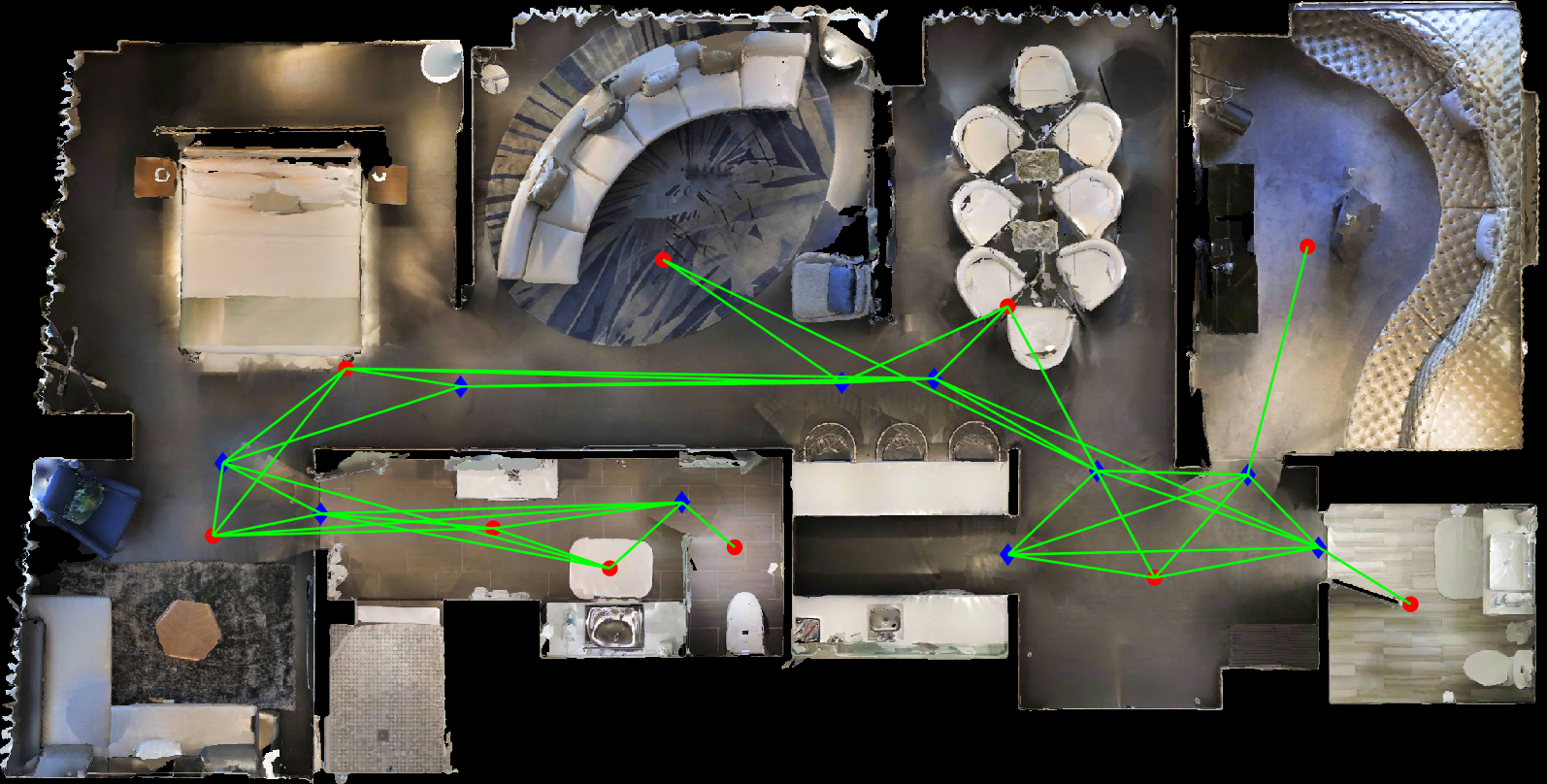}};
\node[below left = -0.1cm and 0.8cm of i3, align=center] {\tiny{Topological Map}};

\draw [->,line width=1.4pt, >=latex] (i2) -- node {} (i3);

\node [right of=i3,  node distance=2.2cm, align=center] (li) {};
\draw [->,line width=1.4pt, >=latex] (i3) -- node {} (li);

\node [dinode,fill=blue, below right = 2.0cm and 1.2cm of i3, align=center] (tn6) {}; 
\node [dinode,fill=blue, below right = 0.9cm and 0.6cm of tn6, align=center] (tn3) {}; 
\node [dinode,fill=blue, above right = 0.8cm and 0.1cm of tn6, align=center] (tn5) {}; 
\node [dinode,fill=blue, above right = -0.5cm and 1.0cm of tn5, align=center] (tn8) {}; 
\node [dinode,fill=blue, below right = 0.18cm and 0.8cm of tn8, align=center] (tn7) {};
\node [dinode,fill=blue, above of= tn8,  node distance=0.8cm, align=center] (tn9) {}; 
\node [dinode,fill=blue, above right = 0.4cm and 0.9cm of tn9, align=center] (tn10) {}; 
\node [dinode,fill=blue, below right = 0.52cm and 0.5cm of tn3, align=center] (tn1) {};
\node [dinode,fill=blue, above right = 0.23cm and 0.6cm of tn1, align=center] (tn2) {}; 
\node [dinode,fill=blue, above right = 0.7cm and 0.6cm of tn2, align=center] (tn4) {}; 

\node [obnode,fill=red, above right = 0.6cm and 0.7cm of tn8, align=center] (n11) {}; 
\node [obnode,fill=red, above left = 0.55cm and 0.6cm of tn5, align=center] (n14) {}; 
\node [obnode,fill=red, above right = 0.9cm and 0.4cm of tn10, align=center] (n16) {}; 
\node [obnode,fill=red, below left = 0.6cm and 0.6cm of tn6, align=center] (n17) {}; 
\node [obnode,fill=red, above right = 0.1cm and 0.7cm of tn4, align=center] (n18) {}; 
\node [obnode,fill=red, above left = 0.5cm and 0.5cm of tn9, align=center] (n19) {}; 
\node [obnode,fill=red, below right = 0.8cm and 1.2cm of tn4, align=center] (n13) {}; 
\node [obnode,fill=red, below right = 0.25cm and 1.15cm of tn4, align=center] (n15) {}; 
\node [obnode,fill=red, below right = 0.7cm and 0.4cm of tn2, align=center] (n12) {}; 
\node [obnode,fill=red, above left = 0.1cm and 1.2cm of tn1, align=center] (n20) {};

\draw [draw,-] (tn1) -- node[sloped, above, color=blue!80, yshift=-0.5mm] {\tiny{27.3}} (tn2);
\draw [draw,-] (tn1) -- node[sloped, above, color=blue!80, yshift=-0.5mm] {\tiny{-17.8}} (tn3);
\draw [draw,-] (tn1) -- node[sloped, below, color=blue!80, yshift=0.5mm] {\tiny{97.9}} (n12);
\draw [draw,-] (tn1) -- node[sloped, below, color=blue!80, xshift=2.5mm, yshift=0.5mm] {\tiny{15.2}} (n13);
\draw [draw,-] (tn1) -- node[sloped, above, color=blue!80, yshift=-0.5mm] {\tiny{-37.5}} (n20);

\draw [draw,-] (tn2) -- node[sloped, above, color=blue!80, yshift=-0.5mm] {\tiny{-1.7}} (tn4);
\draw [draw,-] (tn2) -- node[sloped, below, color=blue!80, yshift=0.5mm] {\tiny{168.4}} (n12);
\draw [draw,-] (tn2) -- node[sloped, above, color=blue!80, yshift=-0.5mm] {\tiny{10.7}} (n13);
\draw [draw,-] (tn2) -- node[sloped, above, color=blue!80, yshift=-0.5mm] {\tiny{4.8}} (n15);

\draw [draw,-] (tn3) -- node[sloped, below, color=blue!80,  xshift=3.1mm, yshift=0.7mm] {\tiny{-0.5}} (tn5);
\draw [draw,-] (tn3) -- node[sloped, below, color=blue!80, yshift=0.6mm] {\tiny{-0.9}} (tn6);
\draw [draw,-] (tn3) -- node[sloped, above, color=blue!80, yshift=-0.5mm] {\tiny{-171}} (n20);

\draw [draw,-] (tn4) -- node[sloped, above, color=blue!80, xshift=3.1mm, yshift=-0.5mm] {\tiny{175.9}} (n12);
\draw [draw,-] (tn4) -- node[sloped, above, color=blue!80,  xshift=3.6mm, yshift=-0.7mm] {\tiny{137.7}} (n13);
\draw [draw,-] (tn4) -- node[sloped, above, color=blue!80, yshift=-0.5mm] {\tiny{172.2}} (n15);
\draw [draw,-] (tn4) -- node[sloped, above, color=blue!80, yshift=-0.5mm] {\tiny{40.4}} (n18);

\draw [draw,-] (tn5) -- node[sloped, above, color=blue!80, xshift=0.1mm, yshift=-0.9mm] {\tiny{-2.6}} (tn6);

\draw [draw,-] (tn5) -- node[sloped, below, color=blue!80, xshift=0.7mm, yshift=0.8mm] {\tiny{-24.7}} (n14);
\draw [draw,-] (tn5) -- node[sloped, above, color=blue!80, yshift=-0.8mm] {\tiny{-145.4}} (n17);
\draw [draw,-] (tn5) -- node[sloped, above, color=blue!80, xshift=0.7mm, yshift=-0.8mm] {\tiny{-178.3}} (n20);

\draw [draw,-] (tn6) -- node[sloped, above, color=blue!80,  xshift=1.5mm, yshift=-0.6mm] {\tiny{29.5}} (tn8);
\draw [draw,-] (tn6) -- node[sloped, above, color=blue!80, xshift=-1.2mm, yshift=-0.5mm] {\tiny{23.7}} (tn10);
\draw [draw,-] (tn6) -- node[sloped, below, color=blue!80, xshift=0.2mm, yshift=0.5mm] {\tiny{-44.5}} (n14);
\draw [draw,-] (tn6) -- node[sloped, above, color=blue!80, yshift=-0.5mm] {\tiny{-179}} (n20);


\draw [draw,-] (tn7) -- node[sloped, below, color=blue!80, yshift=0.5mm] {\tiny{-43.3}} (tn8);	
\draw [draw,-] (tn7) -- node[sloped, below, color=blue!80, yshift=0.5mm] {\tiny{-18.5}} (tn9);
\draw [draw,-] (tn7) -- node[sloped, below, color=blue!80, yshift=0.5mm] {\tiny{-1.3}} (tn10);
\draw [draw,-] (tn7) -- node[sloped, below, color=blue!80, yshift=0.8mm] {\tiny{8.9}} (n11);
\draw [draw,-] (tn8) -- node[sloped, above, color=blue!80, xshift=-1.7mm, yshift=-0.6mm] {\tiny{1.5}} (tn9);
\draw [draw,-] (tn8) -- node[sloped, below, color=blue!80, xshift=3.9mm, yshift=0.6mm] {\tiny{19}} (tn10);
\draw [draw,-] (tn8) -- node[sloped,below, color=blue!80, xshift=1.8mm, yshift=0.7mm] {\tiny{61.7}} (n11);
\draw [draw,-] (tn8) -- node[sloped, above, color=blue!80, yshift=-0.6mm] {\tiny{-118.5}} (n14);
\draw [draw,-] (tn8) -- node[sloped, below, color=blue!80, xshift=3.8mm, yshift=0.5mm] {\tiny{-154}} (n17);

\draw [draw,-] (tn9) -- node[sloped, above, color=blue!80, xshift=-1.2mm,yshift=-0.7mm] {\tiny{45.6}} (tn10);
\draw [draw,-] (tn9) -- node[sloped, below, color=blue!80, yshift=0.6mm] {\tiny{132}} (n11);
\draw [draw,-] (tn9) -- node[sloped, below, color=blue!80, yshift=0.5mm] {\tiny{-75.2}} (n19);

\draw [draw,-] (tn10) to[out=320,in=0] node[sloped, below, color=blue!80, xshift=0.8mm, yshift=0.6mm] {\tiny{169.5}} (n11);
\draw [draw,-] (tn10)  -- node[sloped, below, color=blue!80, yshift=0.7mm] {\tiny{31.6}} (n16);
\draw [draw,-] (n11) to[out=105,in=80] node[sloped, above, color=blue!80, xshift=0.8mm, yshift=-0.6mm] {\tiny{-118.4}} (n14);
\draw [draw,-] (n12) to[out=-40,in=-10] node[sloped, below, color=blue!80, yshift=0.5mm] {\tiny{-1.5}} (n15);
\draw [draw,-] (n12) to[out=190,in=-25] node[sloped, below, color=blue!80, yshift=0.5mm] {\tiny{-51.5}} (n20);
\draw [draw,-] (n13) -- node[sloped, above, color=blue!80, yshift=-0.7mm] {\tiny{-161}} (n15);

\node[draw=cyan,circle, inner sep=7.6pt, above right=-0.32cm and -0.19cm of n16.north west] (mfc) {};

\node[above of= mfc,node distance=0.15cm]{\tiny{scene $2$}};


\node [right of=tn10,  node distance=3.55cm, align=center] (i4) {\colorbox{cyan!60}{\includegraphics[width=.15\textwidth, height=.08\textwidth]{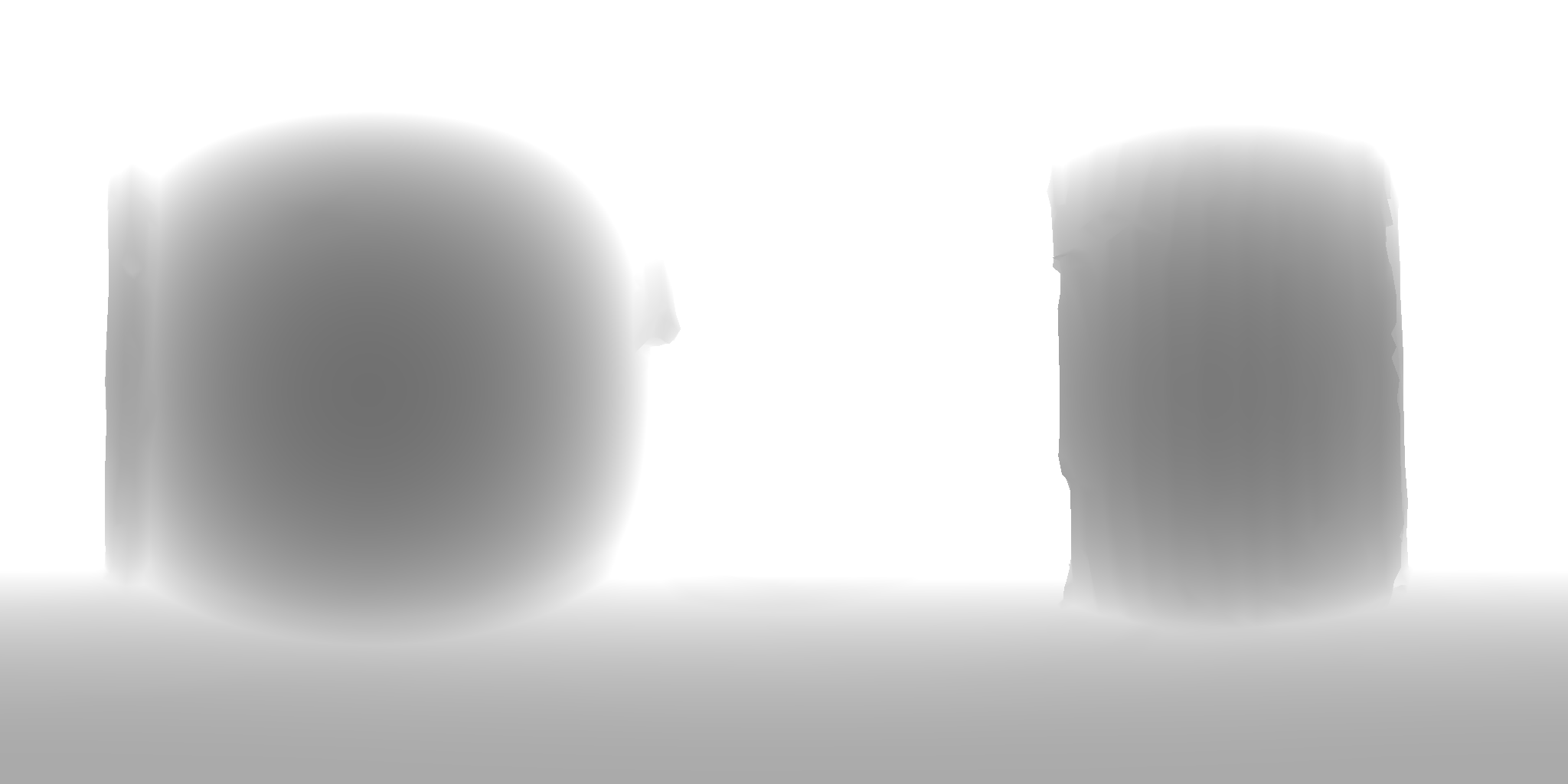}}};
\node [above left = -1.9cm and -3.2cm of i4] (i5) {\colorbox{cyan!60}{\includegraphics[width=.15\textwidth, height=.08\textwidth]{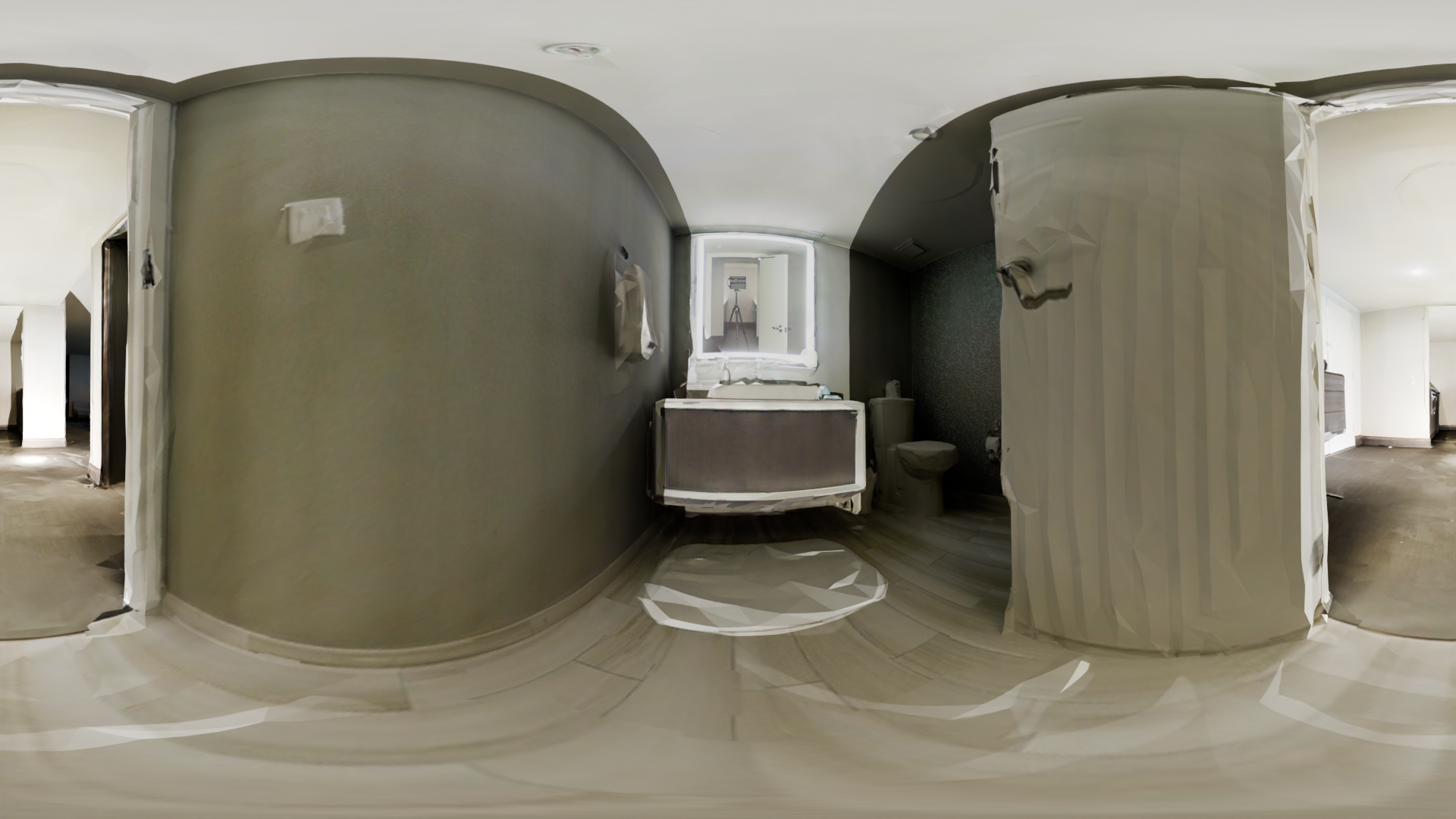}}};


\node [right of=n16 ,node distance=2.93cm, align=center] (arrow) {\quad\quad \footnotesize{\ Panorama RGB-D Scene}};

\node [above right= -0.8cm and 1.83cm of n16, align=center, rotate=70] (arrow1) {\includegraphics[width=.045\textwidth, height=.09\textwidth]{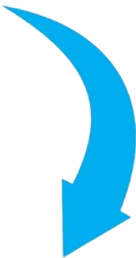}};

\node [right of=i5 ,node distance=0.8cm, align=center] (i5i) {};
\node [block, below of=i5,  node distance=1.4cm, align=center, fill=magenta!3] (od) {\footnotesize{Salient Object} \footnotesize{Detection}};
\node [below of=i5,  node distance=0.65cm, align=center] (i5p1) {}; \node [below of=od,  node distance=0.7cm, align=center] (i5p2) {}; 
\draw [draw,->,line width=1.6pt, >=latex] (i5p1) -- node {} (od); \draw [draw,->,line width=1.6pt, >=latex] (od) -- node {} (i5p2);
 

\node [obnode,fill=White, below of=i5i ,node distance=2.5cm, align=center] (p14) {}; 
\node [obnode,fill=Gray, left of=p14, node distance=1.8cm, align=center] (p24) {}; 
\node [obnode,fill= Magenta, below of= p24, node distance=1.0cm, align=center] (p34) {}; 
\node [obnode,fill=SpringGreen, below of= p14, node distance=1.4cm, align=center] (p44) {}; 

\node [obnode,fill=cyan, below left = 0.3cm and 0.7cm of p14, align=center] (p54) {}; 
\node [obnode,fill=Thistle, below right = 1.1cm and 0.4cm of p14, align=center] (p64) {}; 

\draw [draw,-] (p14) -- node[sloped, below, color=red, yshift=0.5mm] {\tiny{0.124}} (p24);
\draw [draw,-] (p14) -- node[sloped, below, color=red, yshift=0.5mm] {\tiny{0.154}} (p54); \draw [draw,-] (p14) to[out=100,in=115] node[sloped, above, color=red, yshift=-0.5mm] {\tiny{0.23}} (p34);

\draw [draw,-] (p24) -- node[sloped, below, color=red, xshift=-0.4mm,yshift=0.5mm] {\tiny{0.283}} (p34);
\draw [draw,-] (p24) -- node[sloped, below, color=red, xshift=0.6mm,yshift=0.5mm] {\tiny{0.188}} (p54);
\draw [draw,-] (p34) -- node[sloped, below, color=red, yshift=0.5mm] {\tiny{0.329}} (p44);
\draw [draw,-] (p34) -- node[sloped, below, color=red, yshift=0.5mm] {\tiny{0.350}} (p54);
\draw [draw,-] (p44) -- node[sloped, below, color=red, yshift=0.5mm] {\tiny{0.219}} (p54);
\draw [draw,-] (p44) -- node[sloped, below, color=red, yshift=0.3mm] {\tiny{0.145}} (p14);
\draw [draw,-] (p54) -- node[sloped, below, color=red, yshift=0.1mm] {\tiny{0.675}} (p64);


\node[below=1.0cm of p54] (label) {\footnotesize{$360^{\circ}$ Saliency Graph $G$}};


\node [dash_blk, draw=black!50, inner sep=0.0pt, xshift=1.6pt, yshift=5.0pt, fit = (i4) (i5) (od) (p44) (p64) (label), minimum height=5.9cm] (bsg) {};

\node [below of= i1,  node distance=4.2cm, align=center] (qu) {\hspace{0.7cm}\includegraphics[width=.17\textwidth, height=.09\textwidth]{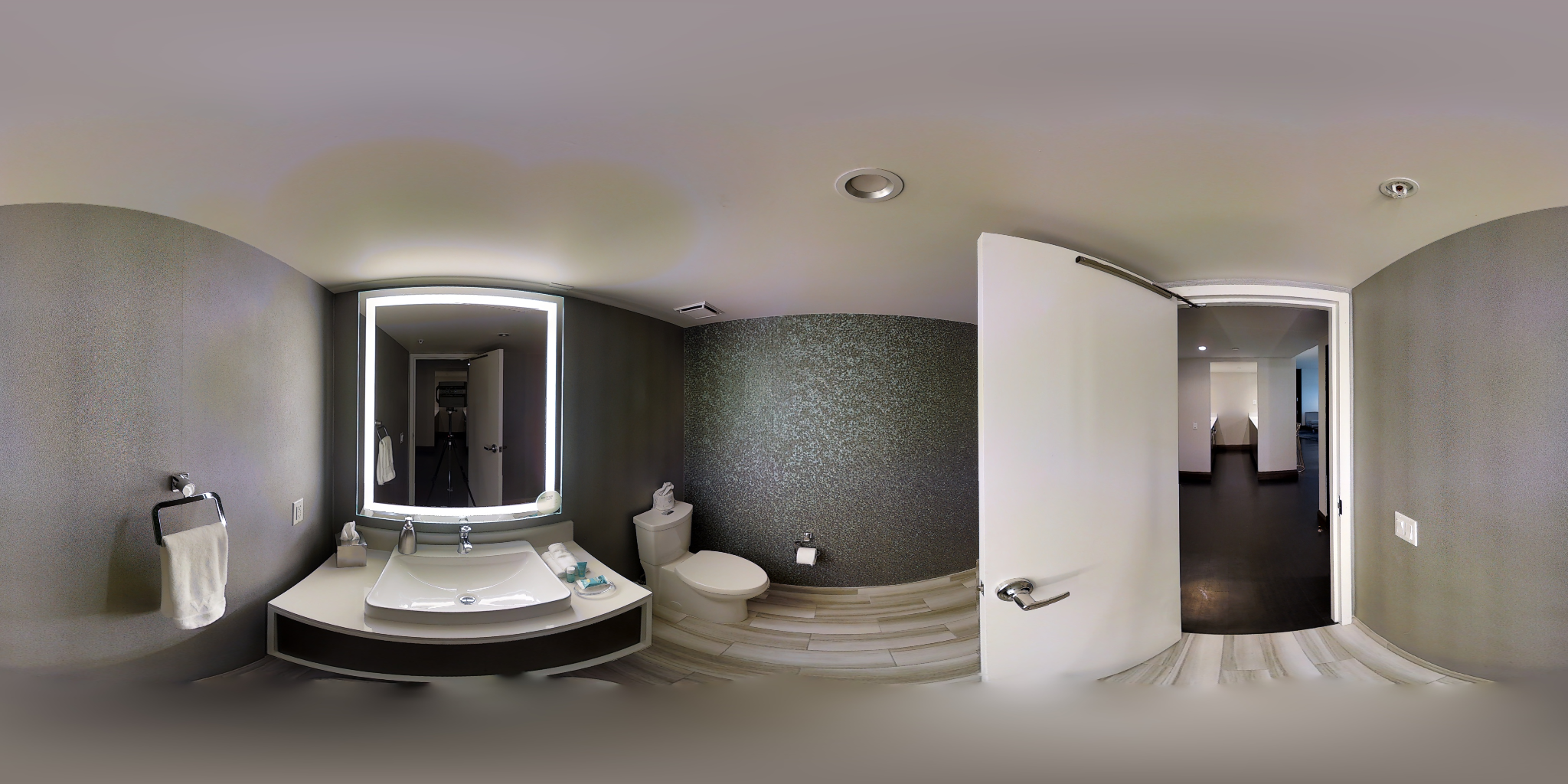}\\ \footnotesize{$360^{\circ}$ Query Scene}};

\node [left of=qu,  node distance=1.4cm, align=center, rotate=90] (B) {\footnotesize{(b) Localization}};
\node [above of=B,  node distance=3.9cm, align=center, rotate=90] (A1) {\footnotesize{(a) Queryable Topological Map Generation}};

\node [right of=qu,  node distance=1.8cm, align=center] (so1) {};

\node [block, right of=so1,  node distance=1.8cm, align=center, fill=cyan!3] (gs) {\footnotesize{$360^{\circ}$ Saliency}\\ \footnotesize{Graph Generation}};
 \draw [draw,->,line width=1.6pt, >=latex] (so1) -- node {} (gs);

\node [right of=gs,  node distance=1.74cm, align=center] (pp2) {};
 \draw [draw,->,line width=1.6pt, >=latex] (gs) -- node [above] {} (pp2);
 
\node [above right =0.1cm and 2.4cm of gs, align=center] (p114) {}; 
\node [left of=p114, node distance=1.8cm, align=center] (p24) {}; 
\node [obnode,fill= Magenta, below of= p24, node distance=1.0cm, align=center] (p34) {}; 
\node [obnode,fill=SpringGreen, below of= p114, node distance=1.2cm, align=center] (p44) {}; 

\node [obnode,fill=cyan, below left = 0.2cm and 0.7cm of p114, align=center] (p554) {}; 
\node [obnode,fill=Thistle, below right = 0.6cm and 0.1cm of p114, align=center] (p64) {}; 


\draw [draw,-] (p34) -- node[sloped, below, color=red, yshift=0.5mm] {} (p44);
\draw [draw,-] (p34) -- node[sloped, below, color=red, yshift=0.5mm] {} (p554);
\draw [draw,-] (p44) -- node[sloped, below, color=red, yshift=0.5mm] {} (p554);
\draw [draw,-] (p554) -- node[sloped, below, color=red, yshift=0.1mm] {} (p64);
\node [below of=p554,  node distance=1.2cm, align=center] (gs1) {\footnotesize{Query Graph}};

\node [block, fill=gray!10, right of=gs,  node distance=5.8cm, align=center] (cg) {\footnotesize{Scene}\\ \footnotesize{Localization} };

\node [right of=gs,  node distance=4.1cm, align=center] (gs1) {};
 \draw [draw,->,line width=1.6pt, >=latex] (gs1) -- node [above] {} (cg);
 
\node [above of=cg,  node distance=1.4cm, align=center] (pta) {\footnotesize{Queryable Topological Map $T$}};

\draw [draw, ->, line width=0.7pt, >=latex, dashed] (pta) ++(0,-0.3) -- node [above] {} (cg);
\node [right of=pta,  node distance=0.15cm, align=center] (ptaa) {};
\draw [draw, ->, line width=0.7pt, >=latex, dashed] (cg) ++(0.15,0.4)-- node [above] {} (ptaa) ++(0,-0.3);



\node [right of=cg,  node distance=2.4cm, align=center] (cate) {\footnotesize{Current}\\ \footnotesize{Location $\mathcal{V}_g$:} \\ \footnotesize{`scene 2'} };
 \draw [draw,->,line width=1.6pt, >=latex] (cg) -- node {} (cate);

\node [block, fill=gray!10, right of=cg,  node distance=4.5cm, align=center] (ps) {\footnotesize{Positioning} };

 \draw [draw,->,line width=1.6pt, >=latex] (cate) -- node [above] {} (ps);
 
\node [right of=ps,  node distance=1.3cm, align=center] (ps1) {};
 \draw [draw,->,line width=1.6pt, >=latex] (ps) -- node [above] {} (ps1);

 \node [right of=ps1,  node distance=0.3cm, align=center] (ps2) {\footnotesize{Pose}}; 
 
\node [obnode,fill=Magenta, below right =0.2cm and 1.8cm of qu, align=center] (o1) {};
\node[above left=-0.65cm and -0.42cm of o1.north west] (l2) {\textbf{\tiny{sink,}}};
\node [obnode,fill=White, right of=o1, node distance=0.6cm, align=center] (o2) {};
\node[above left=-0.65cm and -0.4cm of o2.north west] (l3) {\textbf{\tiny{toilet,}}};
\node [obnode,fill=Gray, right of=o2, node distance=1.1cm, align=center] (o3) {};
\node[above left=-0.65cm and -0.72cm of o3.north west] (l4) {\textbf{\tiny{bathroomvanity,}}};
\node [obnode,fill=SpringGreen, right of=o3, node distance=1.1cm, align=center] (o4) {};
\node[above left=-0.65cm and -0.4cm of o4.north west] (l4) {\textbf{\tiny{mirror}}};

\node [obnode,fill=cyan, right of=o4, node distance=1.1cm, align=center] (o5) {};
\node[above left=-0.65cm and -0.4cm of o5.north west] (l4) {\textbf{\tiny{handtowel}}};
\node [obnode,fill=Thistle, right of=o5, node distance= 0.9cm, align=center] (o6) {};
\node[above left=-0.65cm and -0.38cm of o6.north west] (l4) {\textbf{\tiny{door}}};

\end{tikzpicture}
\vspace{-0.8cm}
\caption{ Overall framework of the proposed method for \textit{Queryable Topological Map Generation}, and \textit{Scene Localization}. For \textit{Queryable Topological Map Generation}, a building is processed to extract the floor map (a 2D layout of the building), the region semantic mask (a segmentation mask labeling larger semantic regions such as rooms and hallways), and the object semantic mask (a finer-grained segmentation mask labeling individual objects) to construct a queryable topological map $T$ that facilitates the navigation path/topological connection between the scene locations $\mathcal{V}$ (nodes in red and scene transition nodes in blue). At each scene location in $\mathcal{V}$, an orientation-rich $360^{\circ}$ saliency graph $G$ is generated from an RGB-D panorama image as a scene representation. In \textit{Scene Localization}, a query panoramic RGB-D image of a scene is localized within the topological map $T$. For this, the query $360^{\circ}$ saliency graph is constructed from the query image and matched against the precomputed saliency graphs $G$ associated with the map nodes to extract the location with the highest similarity above a minimum threshold. Afterwards, positioning is performed by aligning the $360^{\circ}$ saliency graphs of query and reference scenes to estimate the directions for navigation towards the target destination on the topological map.
}
\label{fig:framework}
\vspace{-5mm}
\end{figure*}

\vspace{-0.1cm}
\section{Proposed Method}
\label{sec:method}
\vspace{-0.1cm}

In this section, we describe the proposed $360^{\circ}$ saliency graph representation of the scene map and its application in localization and navigation. The proposed framework comprises two major stages, namely \textit{Queryable Topological Map Generation}, and \textit{Localization for Navigation} as shown in Fig. \ref{fig:framework}. Firstly, we create a queryable topological map consisting of $360^{\circ}$ saliency graphs as scene representation. Then, subsequently, the generated map is utilized for scene localization to facilitate navigation, as discussed in the following subsections. 

\vspace{-0.2cm}
\subsection{Queryable Topological Map Generation}
\vspace{-0.1cm}
A queryable topological map is constructed by first generating a topological map $T$, where each node $\mathcal{V}$ represents a distinct location (e.g., $3$D points) such as a room, hallway, or intersection/transition, and edges represent navigable connections between these locations. At each node $\mathcal{V}$, the environment is captured via a panoramic RGB-D image and represented by its summary called the $360^\circ$ saliency graph $G$, which provides a $360$-degree understanding of the surrounding space or scene. The queryable map supports interactive querying, enabling tasks such as determining the shortest path between two locations or identifying the current location. The $360^\circ$ saliency graph $G$ offers a compact representation of the scene by extracting meaningful information and removing redundant content within the panoramic view. This enhances the efficiency of querying and localization, making the navigation system more robust, scalable, and efficient. 

\subsubsection{Topological Map}

To create a topological map $T$ of an indoor environment, we segment distinct scene points $r_j$ as nodes for this topological map. 
In addition to scene points $r_j$, we include scene transition points $t_i$ as a part of the topological map $T$. Transition nodes $t_i$ are defined as scene points located on the boundaries where the scene changes on the region semantic mask and lie on obstruction-free areas of the floor map. We detect these scene transition nodes $t_i$ using boundary connectivity algorithm \cite{gonzalez2009digital}. The region semantic mask and floor map are generated using the method given in \cite{kim2022habitatmap}. In addition, object semantic mask\cite{kim2022habitatmap} is utilized to find the obstruction-free area within the scene. We finally construct the topological map  $T=(\mathcal{V}, \mathbf{A})$ as a graph with nodes $\mathcal{V}= \{t_i\} \bigcup \{r_j\}$ and the adjacency matrix $\mathbf{A}$ which contains the topological connection between the nodes in $\mathcal{V}$, estimated using the visibility graph \cite{lee2021visibility} framework. 
Figure \ref{fig:framework} shows an example of a generated topological map for a given building. The figure shows that the connections/edges, defined in $\mathbf{A}$, successfully create the traversable paths within a building without any obstruction. In addition, the relative azimuth angle $\phi$ between the key locations $v_p, v_q \in \mathcal{V}$ of the topological map $T$ is also computed as shown in Fig. \ref{fig:framework}. Next, the scene at each node $v$ of $T$ is encoded as a $360^{\circ}$ saliency graph, as discussed in the next subsection.

\subsubsection{$360^{\circ}$ Saliency Graph }

For a given panorama image at each node $v \in T$, we first detect objects within the scene/image using the method described in \cite{guerrero2020s}. Subsequently, the salient objects are detected using the volumetric saliency \cite{meena2024volumetric}. A saliency graph $G=(\mathcal{N}, \textbf{E}, S)$ is then constructed using salient objects $\{\mathcal{O}_{b}\}_{b=1}^n \in \mathcal{N}$ with labels $\{l_{b}\}_{b=1}^n$ as nodes and the contextual relationship between them as edges $e_{a,b} \in E(a,b)$. The ${S}$ represents the volumetric saliency score of the nodes. The edge (contextual relation) between the salient objects ${O}_{a}$ and ${O}_{b}$ is evaluated based on geometry, i.e., spatial distance $d_{s}(a,b)$ and mean dimensions $l_a$ and $l_b$ of objects ${O}_{a}$ and ${O}_{b}$, respectively, as in \eqref{eq:rpred}. We further enrich the saliency graph with semantic information encoded by assigning a weight $w_{i,j} \in \mathbf{W}$; $w_{i,j}= \sqrt{{S}_{O_a} {S}_{O_b}}$ to each edge as in \cite{meena2024indoor}.
\vspace{-1ex}
\begin{equation}
E(a,b)=
\begin{cases}
 1; & \hspace{0.1ex} \textrm{if } \min(l_{a},l_{b})>d_{s}(a,b)\\
 0; & \hspace{0.1ex} \textrm{Otherwise}
\end{cases}
\label{eq:rpred}
\end{equation}
Next, we further enhance the scene's saliency graph $G$ at node $v \in T$ by adding geometric (orientation) information to make it a $360^{\circ}$ saliency graph $\mathcal{G}$. The orientation information at each node of $360^{\circ}$ saliency graph $G$ is encoded by adding the orientation (Polar angle $\hat{\phi}$, and Azimuthal angle $\theta$) of the object's centroid in the spherical coordinates system with respect to the selected axis at corresponding scene node $v \in T$. While the polar angle $\hat{\phi}$ is computed from up (Z) direction, the azimuthal angle $\theta$ is determined with respect to X direction. Figure \ref{fig:360graph}(A) illustrates the selected reference axis $\bar{l}$ and orientation angles $\theta$ in a $360^{\circ}$ saliency graph representation of a sample scene. 
This paper focuses only on ground-based navigation. Therefore, orientation is based solely on the azimuthal angle $\theta$. The constructed $360^{\circ}$ saliency graphs are then used for localization of the queried scene to facilitate navigation, as discussed in the subsequent subsection. 

\vspace{-0.23cm}
\subsection{Scene Localization and Positioning in Scene for Navigation}
\vspace{-0.13cm}

Each scene in the queryable topological map is represented as a $360^{\circ}$ saliency graph, enabling saliency-based localization. Previous works such as \cite{meena2024volumetric} and \cite{meena2024indoor} have explored indoor localization using saliency cues. The framework in \cite{meena2024volumetric} utilizes visual, spatial, and directional features but lacks structural encoding, resulting in limited performance. Although \cite{meena2024indoor} integrates visual, contextual, and semantic information, but omits high-level geometric cues such as $360^{\circ}$ orientation. We address these gaps by extending \cite{meena2024indoor} with a $360^{\circ}$ scene representation for improved localization. 

\subsubsection{Scene Localization}

The estimated $360^{\circ}$ saliency graphs $G$, as discussed in the previous subsection, are utilized for the localization in the scene map $T$ using saliency graph-matching \cite{meena2024indoor}. A panorama query scene of an unknown location $v_q$ is first converted to the query saliency graph as $G^q=(\mathcal{N}^q, \textbf{E}^q, {S}^q)$ and subsequently, we select the candidate (possible matched) graphs $\mathcal{G}_{\hat{i}}$ from the $360^{\circ}$ saliency graphs of an indoor environment $\mathcal{G}_{i}$. For candidate graphs $\mathcal{G}_{\hat{i}}$ selection, Jaccard similarity index is computed as $\hat{i} =\arg \max_{i} \ \mathcal{J}(G^q,{\mathcal{G}_{i}})$. Instead of strict matching, all graphs ${\mathcal{G}_{i}}$ with the Jaccard index greater than $0.6\times \max(\mathcal{J}(.))$ (check supplementary) are selected as candidate graphs to increase the robustness of the graph-matching method. Next, using a node alignment matrix $\mathbf{A}$, a candidate subgraph $ {{\breve{\mathcal{G}}}}_{\hat{i}} \subseteq  \mathcal{G}_{\hat{i}}$ is extracted by retaining only matched nodes to $G^q$. The matched nodes in matrix $\mathbf{A}$ can be determined by considering the maximum node-wise similarity as $\max(\Psi({\it{f}}_{a},\tilde{{\it{f}}}_{b}))$. 
Here, $\Psi(.)$ is the similarity metric between the feature vectors ${\it{f}}_{a}$ and $\tilde{{\it{f}}_{b}}$ of $a_{th}$ and $b_{th}$ nodes of $G^q$ and $\mathcal{G}_{\hat{i}}$ respectively. A feature vector consists of an object label, saliency score, and feature vector extracted using a `node2vec' algorithm \cite{grover2016node2vec}. After alignment, $G^q$ and ${{\breve{\mathcal{G}}}}_{\hat{i}}$ are matched via edge's triplet matching \cite{meena2024indoor}.  
The triplet matching score $\mathcal{I}_{\hat{i}}$ between $G^q$ and ${{\breve{\mathcal{G}}}}_{\hat{i}}$ is computed as $\mathcal{I}_{\hat{i}}=\sum \sum (\mathbf{W} \circ \mathbf{M})$. Here, $\mathbf{W}$ and $\mathbf{M}$ represent the edge-weight matrix of $360^{\circ}$ saliency graph and a matrix containing matching edges, respectively. The symbol $\circ$ signifies the \textit{Hadamard} product. The final best-matched graph among all the matched candidate subgraphs ${\breve{\mathcal{G}}}_{\hat{i}}$ is determined by $\mathcal{G}_m =\arg \max_{\hat{i}} \ \mathcal{I}_{\hat{i}}$. The best-matched graph's location $v_p$ on the topological map $T$ is assigned to the query graph's location $v_q$.

\subsubsection{Positioning}

The localization step identifies the location of the query scene on the map $T$. However, the position and orientation of the reference points (estimated node locations $v_p$ in $T$) in the query graph $G^q$ and matched scene candidate graph $\mathcal{G}_m$ do not get aligned as illustrated in Fig. \ref{fig:deviation}. Thus, we align the reference points $c$ and $c'$ of the matched scene and query scene, respectively, by first estimating the required shift using \textit{Position Estimation}, then followed by \textit{Orientation Alignment} as described below.

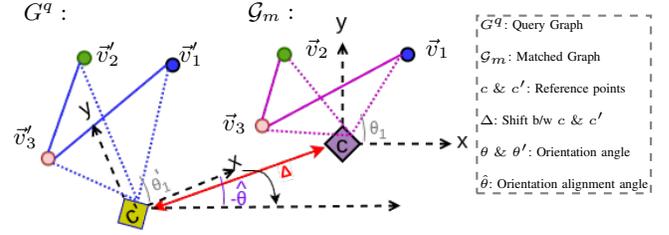
\begin{figure}[!ht]
\centering
\begin{tikzpicture}
\node [input, name=input] {};         
\node [left of =input, node distance=1.8cm] (input1) {};

\node [above left=-1.25cm and 1.5cm of input1] (az) {\includegraphics[width=0.64\columnwidth]{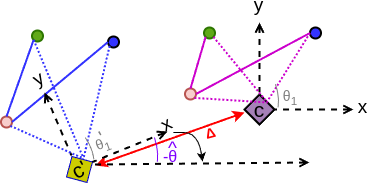}};

\node[above right=-0.8cm and -0.8cm of az] (vec) {\small{$\vec{v}{_1}$}};
\node[above right=-0.8cm and -2.4cm of az] (vec1) {\small{$\vec{v}{_2}$}};
\node[below right=-1.8cm and -3.5cm of az] (vec2) {\small{$\vec{v}{_3}$}};

\node[above left=-0.9cm and -2.4cm of az] (vec1a) {\small{$\vec{v}{'_1}$}};
\node[above left=-0.9cm and -1.3cm of az] (vec2a) {\small{$\vec{v}{'_2}$}};
\node[above left=-2.0cm and -0.2cm of az] (vec3a) {\small{$\vec{v}{'_3}$}};

\node[above left=-0.3cm and -0.6cm of az] (G1) {\small{${G^q:}$}};
\node[right of=G1, node distance=3.0cm,align=center] (G2) {\small{$\mathcal{G}_m:$}};

\node[dash_blk, draw=black!50, below right=-0.2cm and 2.3cm of G2, inner sep=1.0pt, align=left] (legend) {\tiny{$G^q$: Query Graph}\\ \tiny{$\mathcal{G}_m$: Matched Graph}\\ \tiny{$c \ \& \ c'$: Reference points}\\ \tiny{$\Delta$: Shift b/w $c \ \& \ c'$}\\ \tiny{$\theta \ \& \ \theta'$: Orientation angle}\\ \tiny{$\hat{\theta}$: Orientation alignment angle}   };

\end{tikzpicture}
\vspace{-0.4cm}
\caption{Alignment of query graph $G^q$ and matched candidate graph $\mathcal{G}_m$ by estimating shift $\Delta$ and orientation using underlying geometry of reference points ($c$, $c'$) and matched salient objects. } 
\label{fig:deviation}
\vspace{-4mm}
\end{figure}

\vspace{1ex}
\noindent \textbf{A. Position Estimation:}
The reference points $c'$ and $c$ for the query scene graph $G^q$ and matched scene candidate graph $\mathcal{G}_m$, respectively, form a special Euclidean group $SE(3)$. The shift $\Delta$ between reference points $c$ and $c'$ can be estimated using geometric properties of the nodes in the matched scene candidate graph for matched node pair positions $(\vec{v}_i, \vec{v}_i')$ in the $360^{\circ}$ saliency graphs. Theorem 1 describes a framework to estimate the least-square estimate for the shift $\hat{\Delta}$. 

\vspace{-0.2cm}
\begin{thm}
Given the two structures $G_1 := \{\vec{u}_i\}_{i=1}^N $ and $G_2 := \{\vec{v}_i\}_{i=1}^N$ with $N$ components form a special Euclidean group $SE(M)$ with corresponding components pair location $(\vec{u}_i, \vec{v}_i)$. The least square $L2$ estimate of shift $\hat{\Delta}$ between the structure $G_1$ and $G_2$ such that the shifted structure $\tilde{G}_1$ and $G_2$ form a special Orthogonal group $SO(M)$ is given by \eqref{eq:shift}, where $L = \binom{N}{2}$.
\vspace{-0.1cm}
\begin{equation}
        \hat{\Delta} = \frac{1}{2} \mathbf{B}^\dag b,  \quad \mathbf{B}  = \begin{bmatrix}
            \hdots & (\vec{u}_i - \vec{u}_j)  & \hdots
        \end{bmatrix}^T_{M \times L} 
        \label{eq:shift}
\end{equation}

\vspace{-1.0ex}

\scalebox{1.0}{
and $\quad   b  = \begin{bmatrix}
            \vdots \\ ||\vec{u}_i||_2 - ||\vec{u}_j||_2 -  ||\vec{v}_i||_2 + ||\vec{v}_j||_2 \\ \vdots
        \end{bmatrix}_{L \times 1} $
}

\end{thm}

\begin{proof} Let $\mathbb{G}$ be a $SE(M)$ group s.t. structures $G_1, G_2 \in \mathbb{G}$. Thus, $G_2 = \mathcal{T}(G_1)$, where transformation $\mathcal{T}(.)$ can be decomposed as rotation $\mathbf{R}$ and translation or shifting $\Delta$ operation as $\mathcal{T} = [\mathbf{R} \ \  \Delta]$. 

Thus, the shifted structure $\tilde{G}_1$ and $G_2$ are related through a rotation transformation only, and, thus $\tilde{G}_1 := \{\vec{u}_i - \Delta \}_{i=1}^N$  and $G_2$ form a special Orthogonal group $SO(M)$. For $SO(M)$, structures will  satisfy the distance invariance property, i.e., $\mathcal{D}(u_i-\Delta) =  \mathcal{D}(v_i)$ under distance function $\mathcal{D(.)}$. For $L2$ distance function:
\vspace{-0.1cm}
\begin{eqnarray}
    ||\vec{v}_i||_2^2  &=&  ||\vec{u}_i-\Delta||_2^2   \quad    \forall (\vec{u}_i, \vec{v}_i)_{i=1}^N \in (G_1, G_2) \\ 
     ||\vec{v}_i||_2^2  &=&  (\vec{u}_i-\Delta)^T(\vec{u}_i-\Delta) \\ 
     ||\vec{v}_i||_2^2  &=&  ||\vec{u}_i||^2_2 - 2 \Delta^T \vec{u}_i +  ||\Delta||_2^2 \label{eq:Dest}
\end{eqnarray}

Taking two component pairs of $(\vec{u}_i, \vec{v}_i)$,  $(\vec{u}_j, \vec{v}_j)$ in \eqref{eq:Dest} and subtracting to reduce the term $||\Delta||_2^2$:
\vspace{-1mm}
\begin{equation}
    ||\vec{v}_i||_2^2 - ||\vec{v}_j||_2^2   =  ||\vec{u}_i||^2_2 -  ||\vec{u}_j||^2_2 - 2 \Delta^T (\vec{u}_i - \vec{u}_j)
    \label{eq:Dform}
\vspace{-1mm}
\end{equation}

For least square estimate of shift $\hat{\Delta}$, on rearranging and stacking \eqref{eq:Dform} all $L = \binom{N}{2}$ component pair, we obtain:
\vspace{-0.05cm}
\begin{equation}
     \underbrace{ \begin{bmatrix}
            \vdots \\ 
          (\vec{u}_i - \vec{u}_j)^T \\ 
          \vdots 
      \end{bmatrix}}_{\mathbf{B}}\Delta = \frac{1}{2}  \underbrace{ \begin{bmatrix} \vdots \\ ||\vec{u}_i||^2_2 -  ||\vec{u}_j||^2_2 - ||\vec{v}_i||_2^2 + ||\vec{v}_j||_2^2  \\ \vdots \end{bmatrix}}_{b}
    \label{eq:Dform1}
\end{equation}
And thus least square $L2$ estimate of the shift,    $\hat{\Delta}$ is obtained using \textit{Moore-Penrose} inverse $\mathbf{B}^\dag$ as $\hat{\Delta} = \frac{1}{2} \mathbf{B}^\dag b$.

\end{proof}

\vspace{-0.2cm}
\begin{remark}
In the case of the $2D$ navigation task in our work, the query scene graph $G^q$ and matched scene candidate graph $\mathcal{G}_m$ form a $SE(2)$ group. And, the least square $L2$ estimate $\hat{\Delta}$ of the shift between the reference points $c'$ and $c$, for $G^q$ and $\mathcal{G}_m$, respectively, is computed with geometric information of matched salient objects as per the Theorem 1. The shifted query graph $\tilde{G}^q$ is utilized for the orientation estimation as described in the next subsection.
\end{remark}

\vspace{0.6ex}
\noindent \textbf{B. Orientation Alignment:} The shifted query scene graph $\tilde{G}^q$ and matched scene candidate graph $\mathcal{G}_m$ belong to the same special orthogonal group $SO(3)$ as reference points $c'$ and $c$ are aligned. The orientation between these graphs is estimated by estimating the difference in orientation $\theta$ between position vectors $(\vec{v}_i', \vec{v}_i)$ of corresponding salient objects, i.e., matched nodes in $\tilde{G}^q$ and $\mathcal{G}_m$, respectively. Rotation framework such as $X_\alpha Y_\beta Z_\gamma$ with rotation transformation $\mathbf{R}(\theta)$ can be utilized on this $SO(3)$ group to compute the angle $\theta = (\alpha, \beta, \gamma)$  as given in \eqref{eq:avgorientation}. 
\vspace{-0.15cm}
\begin{equation}
       \mathbf{P}  =   \mathbf{R}(\theta)  \mathbf{\tilde{P}}
    \label{eq:avgorientation}
    \vspace{-1mm}
\end{equation}

Here, matrices $\mathbf{P}$ and $\mathbf{\tilde{P}}$ are formed by column-wise stacking of position vectors $\vec{v}_i$ and $\vec{v}_i'$ of matched nodes in $\mathcal{G}_m$ and $\tilde{G}^q$, respectively. For the case of $2D$ navigation, the shifted query scene graph and matched scene candidate graph belong to $SO(2)$, and orientation $\theta$ between these graphs can be simply estimated as the difference between angles of corresponding salient objects from the X-axis as given in \eqref{eq:2Dorientation}. 
\vspace{-0.1cm}
\begin{equation}
       \theta   =   \arccos(\hat{v}.e_1) -  \arccos(\hat{v'}.e_1),  
    \label{eq:2Dorientation}
    \vspace{-1mm}
\end{equation}

Here, $e_1$ denotes the unit vector along the X-axis, and $.$ denotes the dot product. We take a saliency weighted average of the estimated orientation difference between all $N$ pairs of matched nodes of shifted query scene graph $\tilde{G}^q$ and matched scene candidate graph $\mathcal{G}_m$ as orientation alignment angle $\hat{\theta}$ to minimize the error as given in \eqref{eq:avgorientation}.  
\vspace{-0.1cm} 
\begin{equation}
        \hat{\theta} = \frac{\sum_{i=1}^{N} S_i  \theta_i}{\sum_{i=1}^{N} S_i}
    \label{eq:avgorientation1}
\end{equation}

Next, we utilize the Floyd-Warshall \cite{triana2018implementation} algorithm to find the shortest paths between current location $v_p$ and target location $v_q$ on the topological map $T$ to facilitate navigation. For navigation, based upon estimated shift $\hat{\Delta}$ \eqref{eq:shift} and orientation alignment $\hat{\theta}$ \eqref{eq:avgorientation1}, the viewer's current orientation $\theta_c$ at $v_p$ having position $c=(X_c,Y_c)$ is utilized to orient towards the next node in the shortest path to target location $v_q$ on topological map $T$. The current orientation $\theta_c$ refers to the direction of viewer relative to reference direction at node $v_p$.

\begin{figure}[!hbt]
\vspace{-2.5mm}
\centering
\begin{tikzpicture}
\node [input, name=input] {};         

\node [above right=-1.0cm and -13.9cm of input] (ip1) {\includegraphics[width=0.78\linewidth]{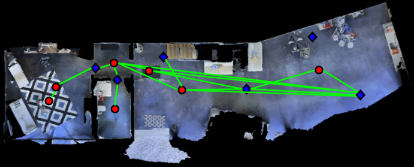}};

\node [below right= -1.1cm and -5.3cm  of ip1, align=center] (tp) {\textcolor{red}{{\textbf{$8$}}}};
\node [below right= -2.5cm and -5.5cm  of ip1, align=center] (tp) {\textcolor{red}{{\textbf{$9$}}}};

\node [below right= -0.16cm and -4.7cm  of ip1, align=center] (tp) {\tiny{Topological Map}};

\node [below of= ip1, node distance=2.3cm, align=center] (node13) {\includegraphics[width=0.25\linewidth]{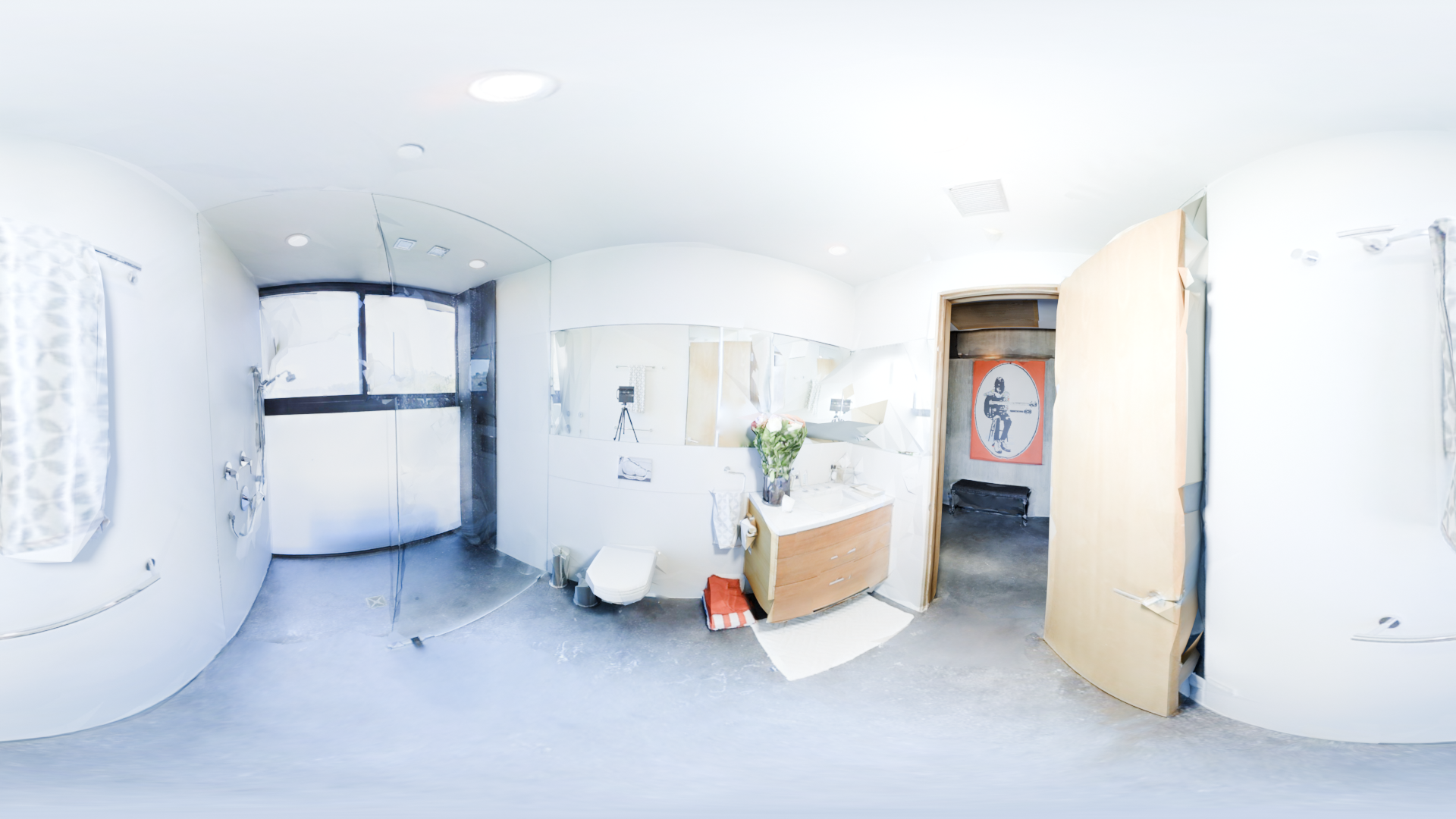} \includegraphics[width=0.25\linewidth]{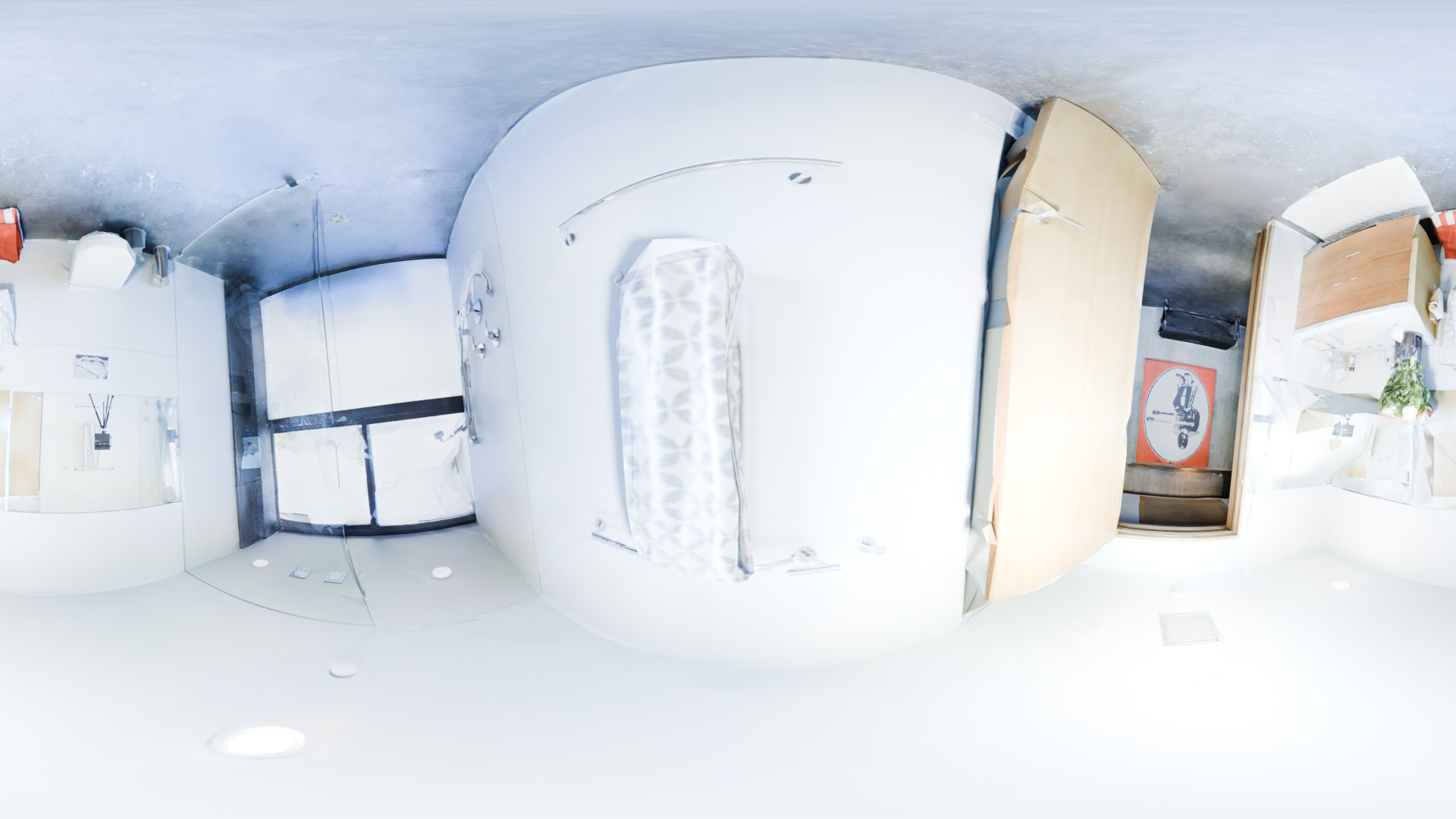} \includegraphics[width=0.25\linewidth]{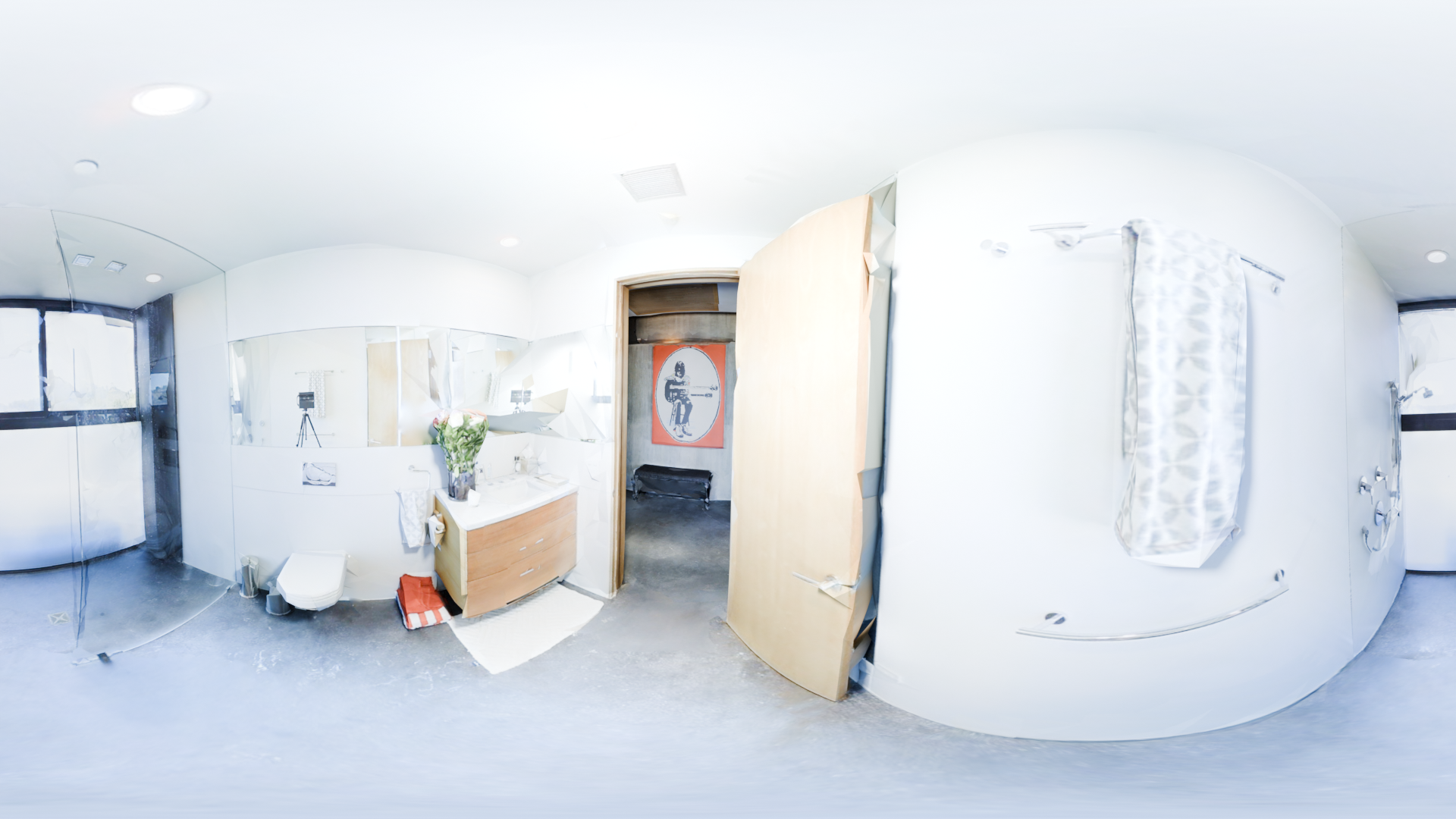}} ;

\node [below left= -0.2cm and -7.6cm  of node13, align=center] (cu) {\tiny{original scene at $v_{8}$} \hspace{0.5cm} \tiny{orientation difference$=22.4$} \hspace{0.4cm} \tiny{orientation difference$=7.9$} };

\node [below of= node13, node distance=1.5cm, align=center] (9_1) {\includegraphics[width=0.25\linewidth]{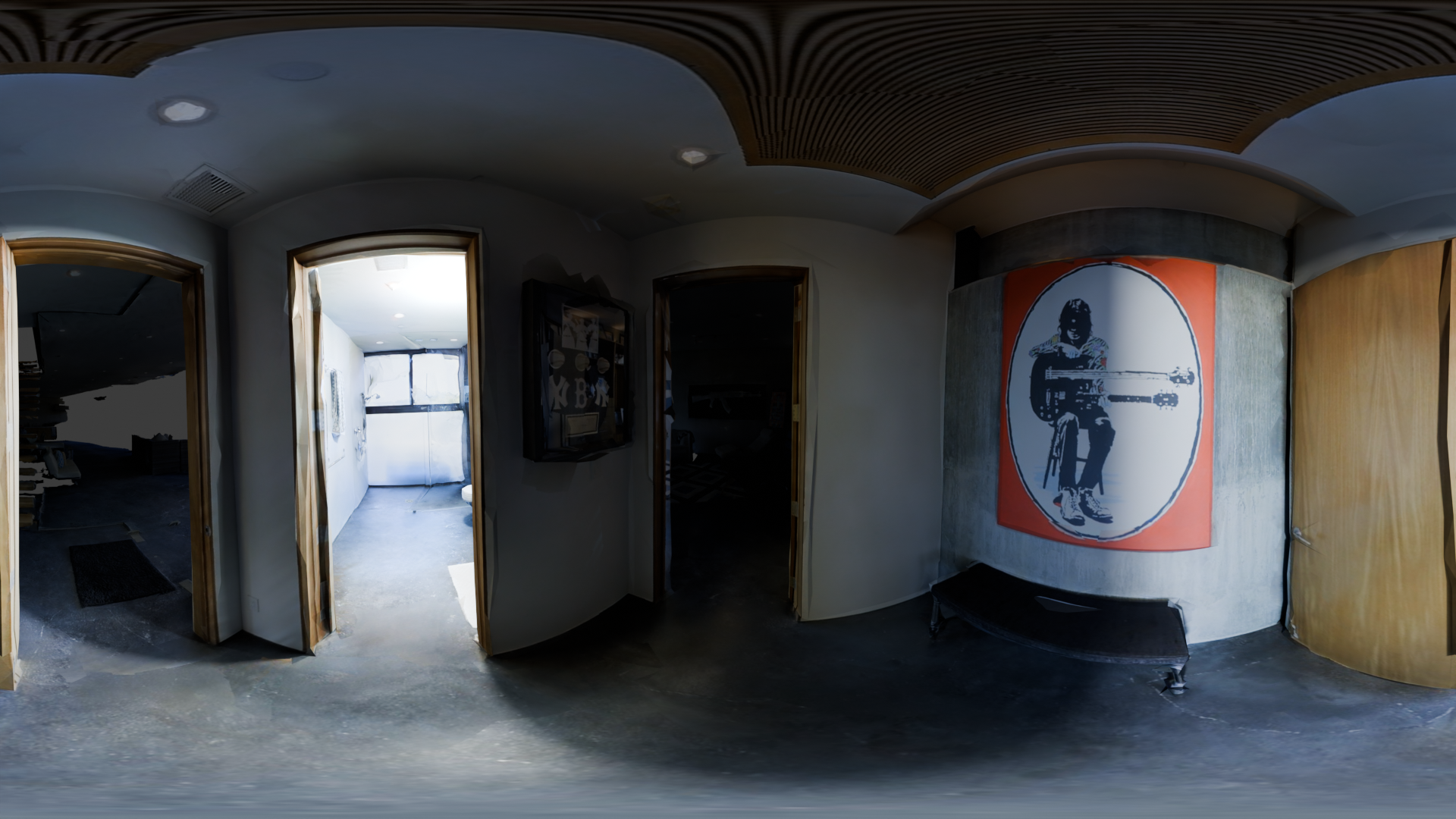} \includegraphics[width=0.25\linewidth]{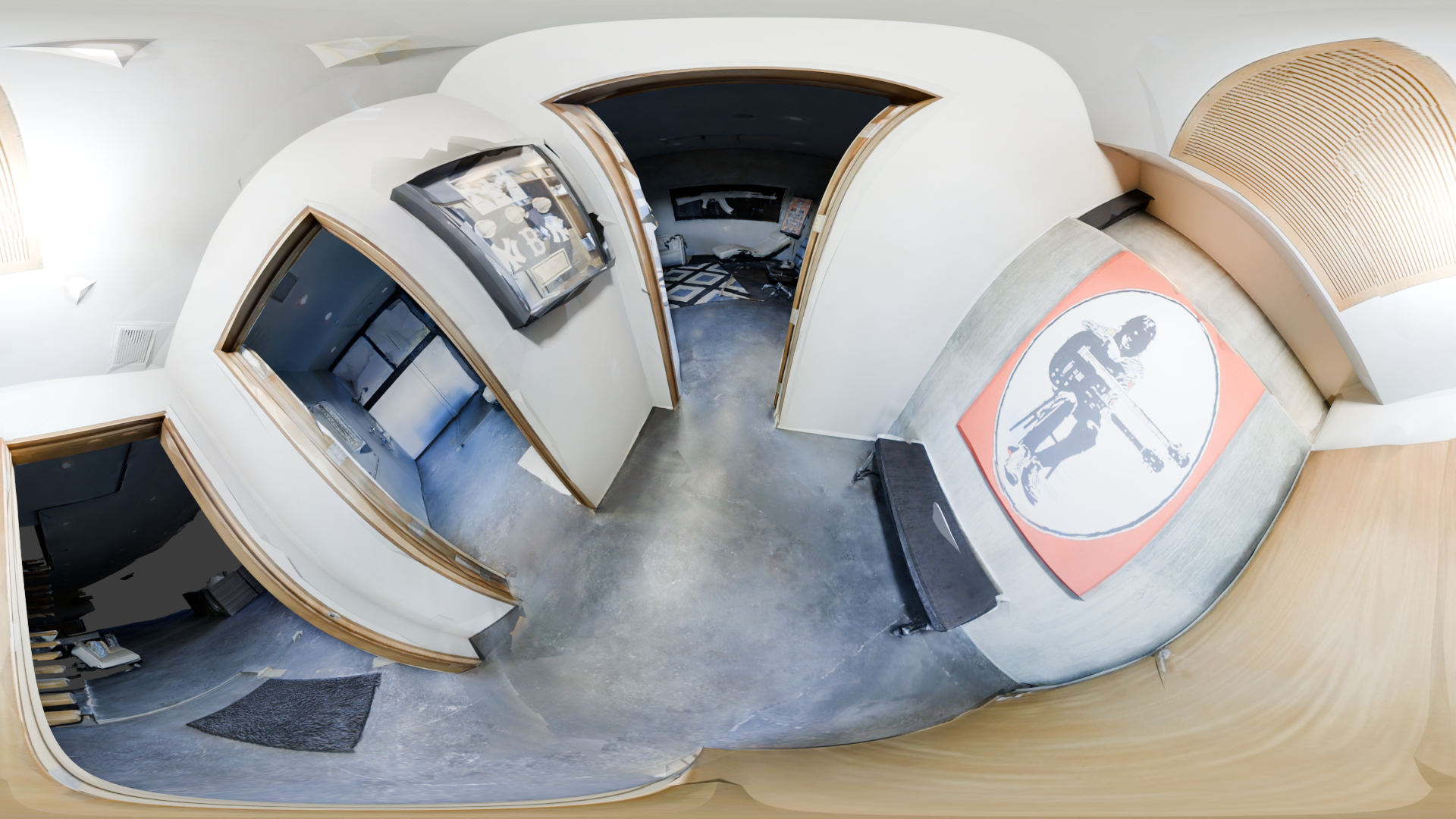} \includegraphics[width=0.25\linewidth]{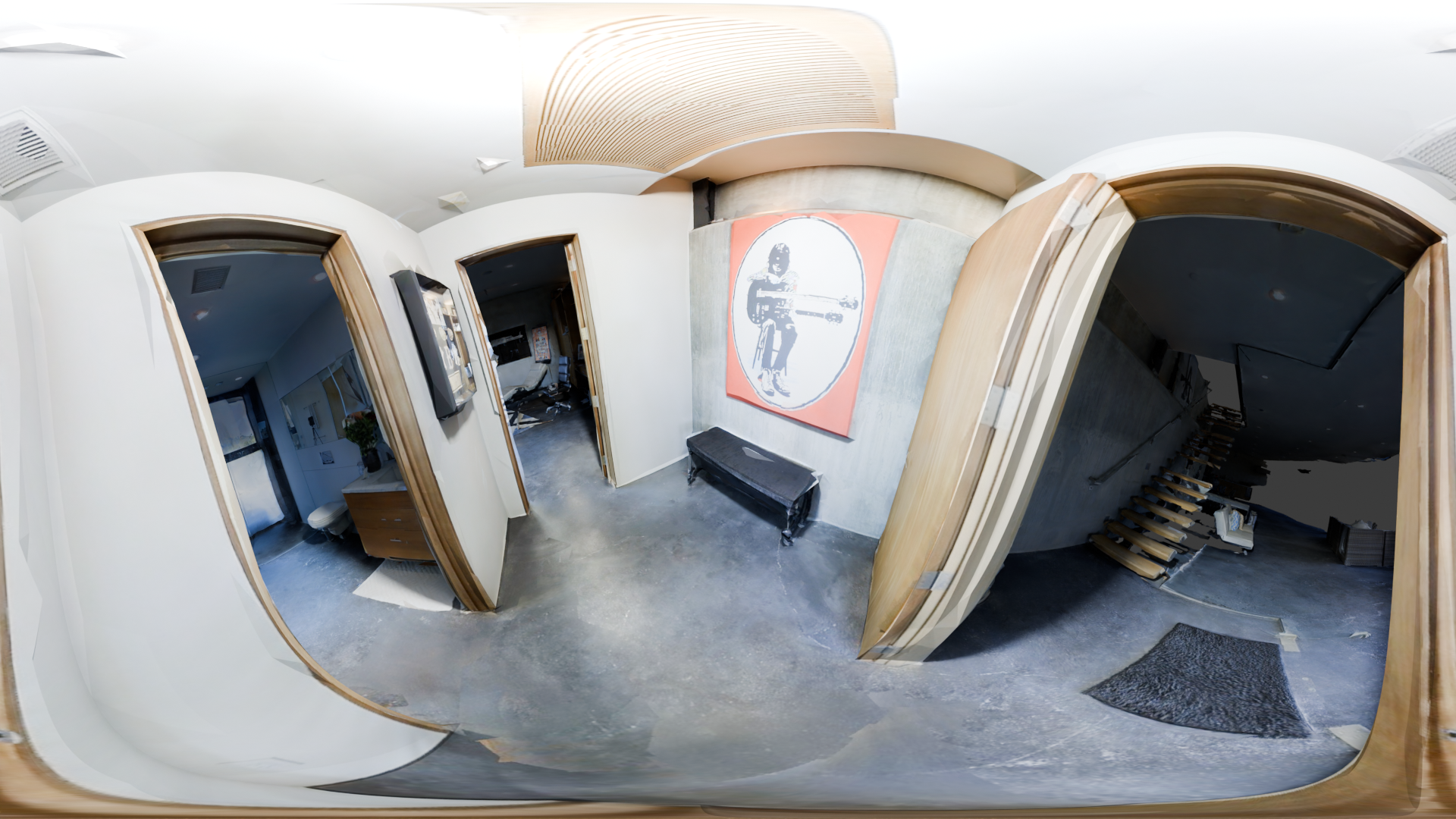}} ;

\node [below left= -0.2cm and -7.6cm  of 9_1, align=center] (cu1) {\tiny{original scene at $v_{9}$} \hspace{0.5cm} \tiny{orientation difference$=13.8$} \hspace{0.4cm} \tiny{orientation difference$=55.8$} };



\end{tikzpicture}
\vspace{-0.5cm}
\caption{Sample panorama images for different viewer positions and orientation difference ($\theta_e$).} 
\label{fig:orien_err}
\vspace{-4.5mm}
\end{figure}


The proposed positioning approach is evaluated for $SE(3)$ transformations. Panorama images are rendered at each original (ground truth) position $c$, as defined by the topological map, and at shifted position $c'$, obtained by applying rotations and translations. Figure \ref{fig:orien_err} presents a few sample panorama images, rendered at ground truth and shifted position using `Blender'\footnote{https://www.blender.org/} (see supplementary), showing the effect of the change in orientation with respect to the viewer position. The estimated translational shift $\hat{\Delta}$ \eqref{eq:shift} and orientation alignment $\hat{\theta}$ \eqref{eq:avgorientation1} are then used to compute the positional error $E_p$ and orientation error $E_\theta$ between the original and shifted images. Table \ref{tab:errinse3} presents the mean absolute error in position $E_p$ and orientation $E_\theta$ for $105$ test images. The results indicate that query scenes captured with a mean positional error of $0.6995$m and an orientation error of $16.59^0$ were used for evaluation. This demonstrates the proposed method’s capability to support efficient navigation despite variations in rotations and translations. As presented in Section \ref{sec:results}, Table \ref{tab:error}, the evaluation indicates that the proposed method can effectively handle pose variations, making it well-suited for navigation scenarios involving $SE(3)$ transformations.  

\begin{table}[!hbt]
    \centering
    \vspace{-2mm}
    \caption{Mean position ($E_{p}$) and orientation ($E_{\theta}$) errors, computed between shifted and original/ground-truth poses. }
    \label{tab:errinse3}
    \vspace{-0.22cm}
     \renewcommand{\arraystretch}{0.8}
\begin{tabular}{c|cc}
\hline
\small{\# of instances} & \small{$E_{p}$} & \small{$E_{\theta}$} \\ \hline
\small{$105$} & \small{$0.6995$m} & \small{$16.59^0$} \\
\hline
 \end{tabular}
\vspace{-3.0mm}
\end{table}

The localization and positioning, after which the viewer executes the necessary rotation to be oriented towards the target node, are utilized for vision-based navigation. After every determined movement towards the target,  the localization process is repeated iteratively for each intermediate node for continuous navigation. Figure \ref{fig:nav} shows an example of navigation on a topological map $T$ of a building. As shown in the figure, at each intermediate node, the viewer begins by identifying the orientation angle $\theta_{\text{c}}$ and moves to the next location to reach the destination $v_g$.

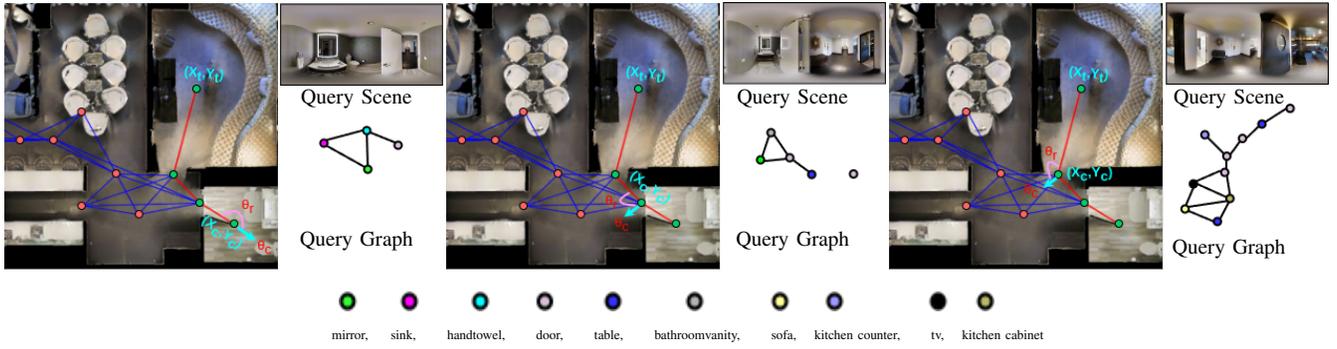
\begin{figure*}[!hbt]
\begin{tikzpicture}
\node [input, name=input] {};         

\node [above right=-1.0cm and -13.9cm of input] (ip1) {\includegraphics[width=0.97\linewidth]{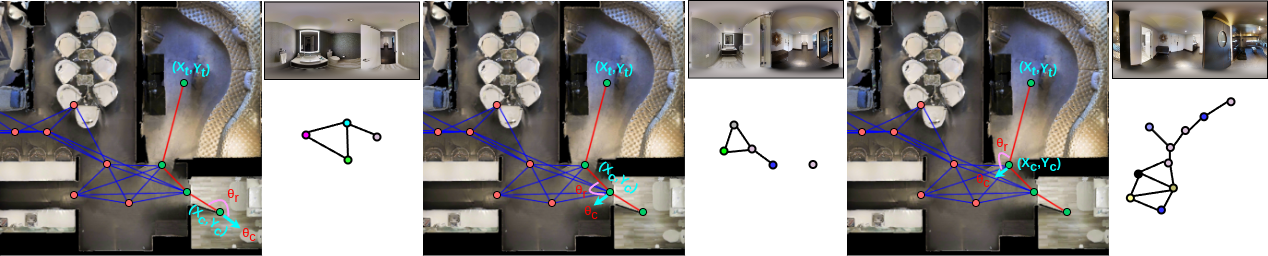}};

\node [below right= -2.65cm and -2.3cm  of ip1, align=center] (cu) {\footnotesize{Query Scene}};
\node [left of= cu, node distance=5.8cm, align=center] (cu1) {\footnotesize{Query Scene}} ;
\node [left of= cu1, node distance=5.8cm, align=center] (cu2) {\footnotesize{Query Scene}} ;

\node [below of= cu, node distance=2.0cm, align=center] (cg) {\footnotesize{Query Graph}} ;
\node [below of= cu1, node distance=1.9cm, align=center] (cg1) {\footnotesize{Query Graph}} ;
\node [below of= cu2, node distance=1.9cm, align=center] (cg2) {\footnotesize{Query Graph}} ;

\node [below of= ip1, node distance=2.2cm, align=center] (iplab) {\includegraphics[width=0.48\linewidth]{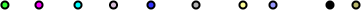}} ;
\node [below right =0.01cm and -9.4cm of iplab, align=center] (label) {\tiny{\quad\quad mirror, \ \quad sink, \quad\quad handtowel, \qquad door, \quad\quad  table, \quad\quad bathroomvanity, \quad\quad sofa, \quad kitchen counter, \quad\quad tv, \quad kitchen cabinet}};

\end{tikzpicture}
\vspace{-0.4cm}
\caption{An example of navigation on topological map. 
First, determine current orientation $\theta_c$ at the current location $(X_c,Y_c)$ using the saliency graph after performing localization. Next, the rotation angle $\theta_r$ between the direction of the topological path and $\theta_c$ is calculated after performing alignment using positioning. The navigation algorithm then adjusts the heading based on this $\theta_r$, ensuring accurate movement towards the next target location $(X_t,Y_t)$ on the shortest path algorithm.
} 
\label{fig:nav}
\vspace{-5mm}
\end{figure*}

\vspace{-0.17cm}
\section{Experiments}
\label{sec:results}
\vspace{-0.13cm}

This section presents the performance analysis of the proposed $360^{\circ}$ saliency graph-based scene localization and navigation method and compares it with state-of-the-art methods.

\noindent \textbf{Implementation Details:} Our experimental setup involves {\it{MATLAB}} {\it{R2022a}} in {\it{Windows 10}} environment with an `Intel Xeon' CPU having $64$ GB RAM. For evaluation, we have utilized \textit{Matterport3d} \cite{chang2017matterport3d} dataset, containing a total of $500$ scenes from $28$ buildings. Each building contains $5$ to $41$ scenes or rooms. Further, we have also utilized `Blender' to get new images of different configurations for robustness evaluation. 

\vspace{-0.22cm}
\subsection{Performance Comparison:}
\vspace{-0.13cm}
We have compared the performance of the proposed method in terms of Localization accuracy and Navigation accuracy, as discussed in the following subsections.

\vspace{0.6ex}
\noindent \textbf{I. Localization Performance:}
We compare the localization performance of the proposed method in terms of accuracy with conventional matching methods including detector-based matching \cite{detone2018superpoint}, \cite{zhao2022alike}, \cite{luo2020aslfeat}, \cite{sarlin2020superglue} and detector-free matching \cite{chen2022aspanformer}, \cite{tang2022quadtree}, \cite{sun2021loftr}, \cite{jiang2021cotr}, and \cite{zhang2023searching}. These methods perform localization by matching image features. The features can also be obtained by detecting keypoints in the images, and matching is done by establishing correspondences between keypoints in a query image and those in a database image. The database image with the highest number of reliable matches is considered the best match, and its known location is used as the estimated location of the query. To evaluate the matches Mean Matching Accuracy (MMA) metric is used. This metric indicates the proportion of correct matches among all matches. Table \ref{tab:comp} lists the MMA$@i$ in percentage for thresholds $i \in (3)$ obtained for SOTA and the proposed method. We observe from the table that detector-free matching \cite{zhang2023searching}, \cite{chen2022aspanformer}, \cite{tang2022quadtree}, \cite{sun2021loftr}, \cite{jiang2021cotr} achieves better MMA$@i$ than detector-based matching \cite{detone2018superpoint}, \cite{zhao2022alike}, \cite{luo2020aslfeat}, \cite{sarlin2020superglue}. 
In contrast, the proposed method results in a significant improvement in localization accuracy in terms of the MMA$@i$ by introducing saliency graph matching. Thus, this reflects that the  $360^{\circ}$ saliency graph efficiently encodes the relevant visual, semantic, and contextual information of the scene.

\begin{table}[!hbt]
\vspace{-2mm}
    \centering
    \caption{Localization performance comparison in terms of MMA$@3$ with state-of-the-art feature matching-based methods. }
    \label{tab:comp}
      \vspace{-0.2cm}
    \renewcommand{\arraystretch}{1.1}
    \resizebox{\columnwidth}{!}{%
\begin{tabular}{c|c||c|c}
\hline
 Methods  &  MMA@3 &  Methods  &  MMA@3\\ \hline
SP \cite{detone2018superpoint}+NN   & $32.43$ & QuadT \cite{tang2022quadtree} & 41.72\\
SP \cite{detone2018superpoint}+SG \cite{sarlin2020superglue} & $29.95$ & LoFTR \cite{sun2021loftr} & 36.07\\
ASLFeat \cite{luo2020aslfeat}+NN   & $35.79$ & COTR \cite{jiang2021cotr}  & $46.07$ \\
ALike \cite{zhao2022alike}+NN  & $34.55$ & GT+SGAM\_COTR \cite{zhang2023searching} & $49.82$ \\
ASpan \cite{chen2022aspanformer} & $37.25$ & Proposed  & $\mathbf{87.8}$\\
\hline
 \end{tabular}
}
{\tiny{(NN: Neural Network)}}
\vspace{-3mm}
\end{table}

We also compare the localization performance of the proposed method with recent scene localization methods, including \cite{kim2023ldl}, \cite{kim2024fully}, \cite{qiao2023march}, \cite{xie2024mfp}, and \cite{meena2024volumetric} in Tab. \ref{tab:loccomp} in terms of Accuracy $Acc=$(\# \text{Correctly Localized Queries})/(\# \text{Total Test Queries}). The table shows that the method LDL \cite{kim2023ldl}, which utilized the distribution of lines with $2$D and $3$D line distance functions, achieves the lowest localization \textit{Accuracy} among all. FGPL \cite{kim2024fully} further increases the localization performance by incorporating the geometry of $2$D-$3$D lines. MiC \cite{qiao2023march} shows superior performance compared to LDL \cite{kim2023ldl} and FGPL \cite{kim2024fully} by incorporating CLIP \cite{radford2021learning} features.
The method MFP-CNN \cite{xie2024mfp} shows an increment in the value of \textit{Accuracy} using multi-scale feature fusion. On including the salient object's volumetric information for localization as in SOVS \cite{meena2024volumetric}, the \textit{Accuracy} increases by $2\%$. The localization using the proposed $360^{\circ}$ saliency graph representation further enhances the accuracy by $4\%$ by leveraging visual, contextual, and semantic information.

\begin{table}[!hbt]
\vspace{-2.5mm}
    \centering
    \caption{Comparison of mean accuracy for scene localization. }
    \label{tab:loccomp}
     \vspace{-0.24cm}
    \renewcommand{\arraystretch}{1.1}
    \resizebox{0.99\columnwidth}{!}{%
\begin{tabular}{c|ccccccc}
\hline
Methods &  LDL\cite{kim2023ldl}  & FGPL\cite{kim2024fully} & MiC\cite{qiao2023march} & MFP-CNN\cite{xie2024mfp} & SOVS\cite{meena2024volumetric} & Proposed  \\  \hline
\textit{Acc} ($\uparrow$)  & $0.35$ & $0.71$ & $0.78$ &$0.80$ & $0.82$ & $\mathbf{0.86}$\\
\hline
 \end{tabular}
}
\vspace{-3.0mm}
\end{table}

We additionally present a qualitative comparison of the proposed $360^{\circ}$ saliency graph-based scene localization method against methods proposed in \cite{kim2023ldl}, \cite{kim2024fully}, \cite{qiao2023march}, \cite{xie2024mfp}, and \cite{meena2024volumetric}. Table \ref{fig:sceneclass} shows the comparison results for a few test images obtained using SOTA and the proposed method. The ground truth label \textit{BiSj} denotes the location or input taken from the $j_{th}$ scene of the $i_{th}$ building. We observe that the proposed method outperforms others as it considers the salient objects as nodes (visual), saliency score (node semantics), weighted triplet (edge semantics), contextual, and geometric information of the complete surrounding (i.e., $360^{\circ}$) scene.

\begin{table}[!hbt]
    \centering
\vspace{-2.5mm}
    \caption{Sample results for scene localization.}
     \label{fig:sceneclass}
     \vspace{-0.22cm}
     \resizebox{0.98\columnwidth}{!}{%
     \begin{tabular}{lllllll}
\hline
\multirow{2}{*}{\begin{tabular}[c]{@{}l@{}} $360^{\circ}$ scenes with labels\end{tabular}} & \multicolumn{5}{c}{Predicted scene labels} \\
 {} & \begin{tabular}[c]{@{}l@{}} {LDL}\\ \cite{kim2023ldl} \end{tabular} & \begin{tabular}[c]{@{}l@{}}{FGPL}\\ \cite{kim2024fully} \end{tabular} & \begin{tabular}[c]{@{}l@{}}{MiC}\\ \cite{qiao2023march} \end{tabular} & \begin{tabular}[c]{@{}l@{}} MFP-\\CNN \cite{xie2024mfp} \end{tabular} & \begin{tabular}[c]{@{}l@{}} {SOVS}\\ \cite{meena2024volumetric} \end{tabular} & {Proposed} \\
\hline
\\
 \begin{tabular}[c]{@{}l@{}}
\begin{minipage}[b]{0.4\linewidth}
 \begin{overpic}[width=\textwidth]{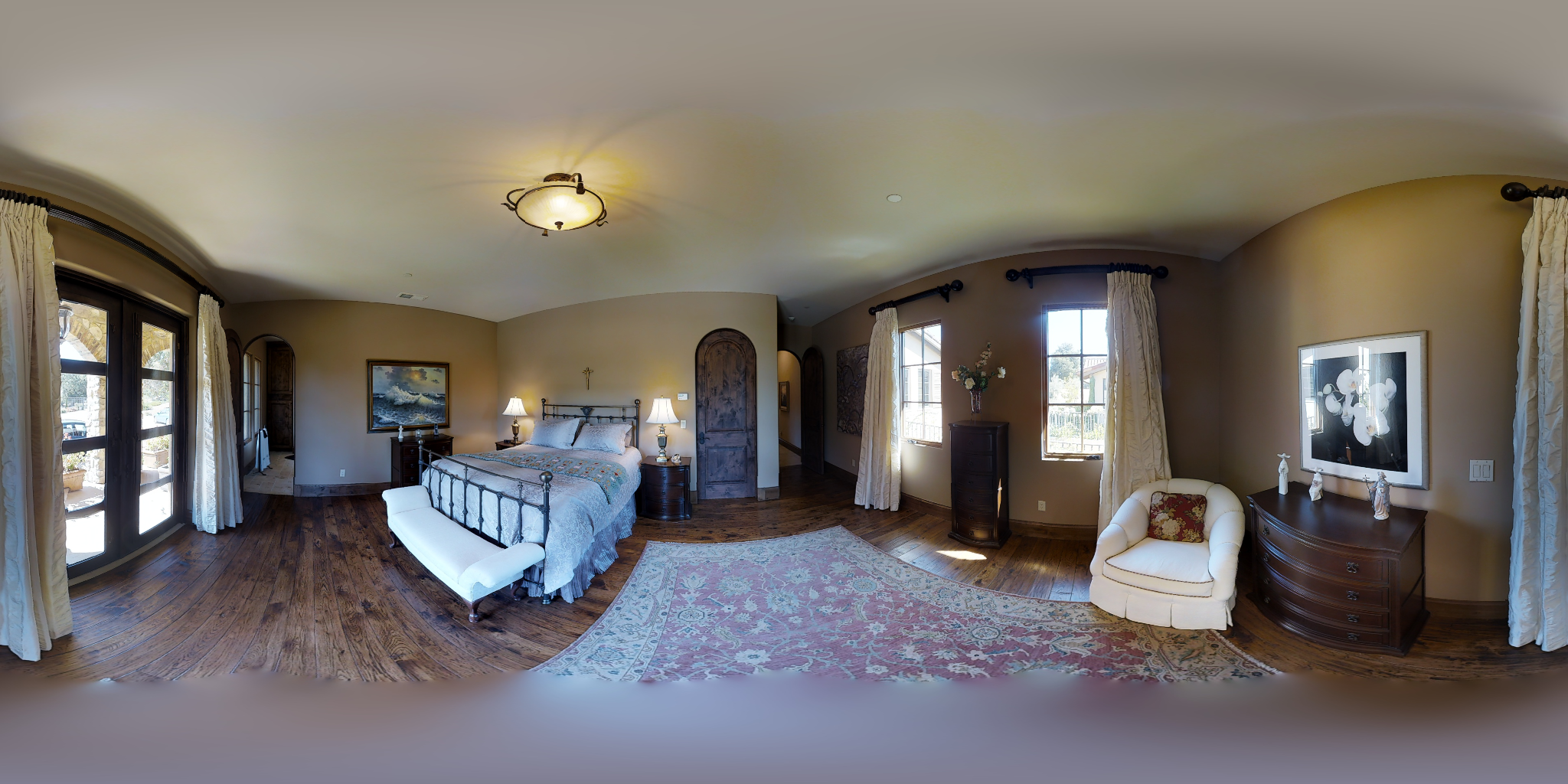}
 \put (2,40) {\fcolorbox{black}{black!15}{\color{blue}{\textbf{B31S22}}}}
 \end{overpic}
 \end{minipage}
 \end{tabular} &  \begin{tabular}[c]{@{}l@{}}B31S26\end{tabular}  &   \begin{tabular}[c]{@{}l@{}}B31S22\end{tabular}  &  \begin{tabular}[c]{@{}l@{}}B31S26\end{tabular}  & \begin{tabular}[c]{@{}l@{}}B31S22\end{tabular}  &  \begin{tabular}[c]{@{}l@{}}B31S25\end{tabular}  &\begin{tabular}[c]{@{}l@{}}B31S22\end{tabular}  \\

\begin{tabular}[c]{@{}l@{}}
\begin{minipage}[b]{0.4\linewidth}
\begin{overpic}[width=\textwidth]{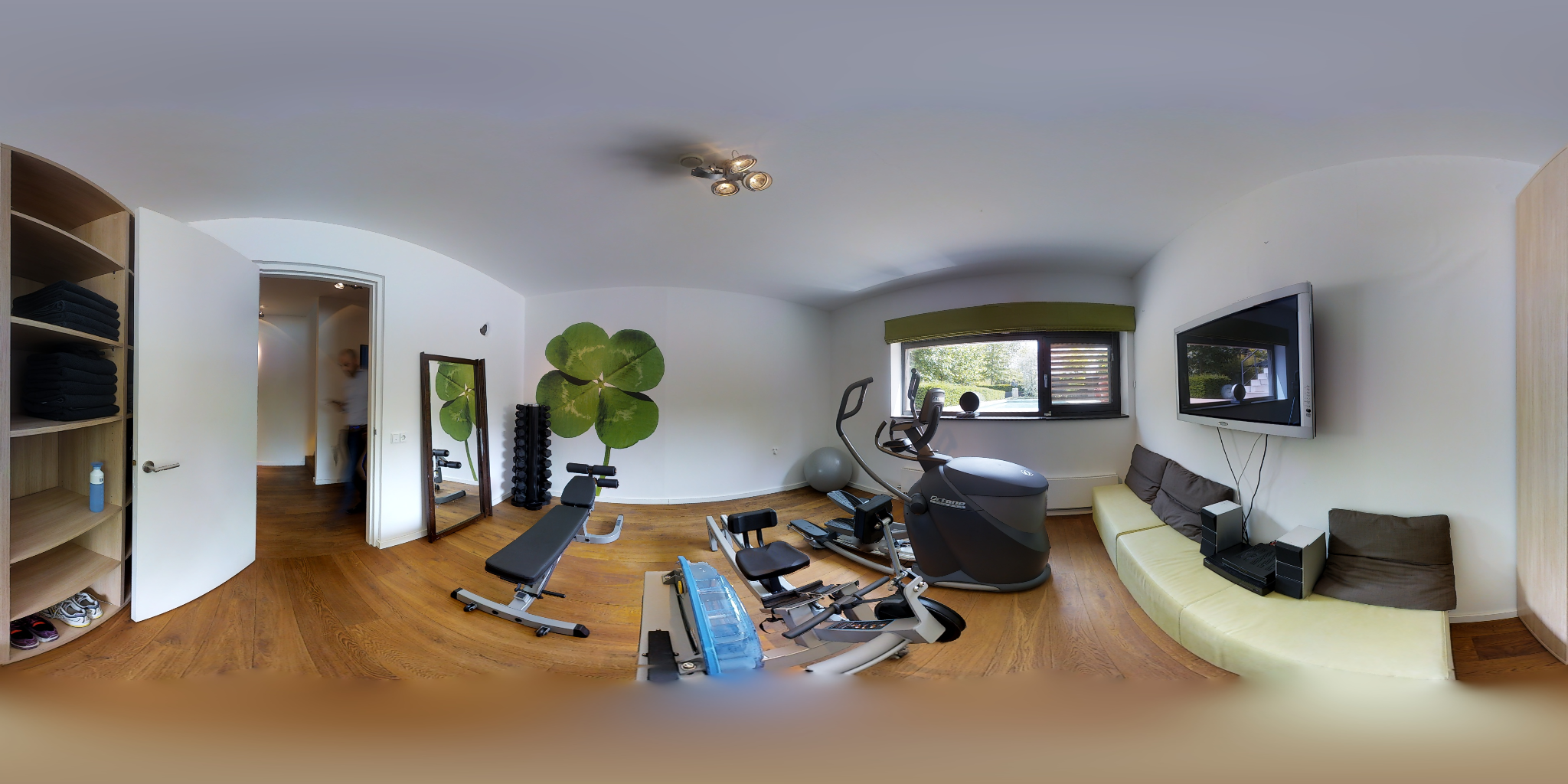}
\put (2,40) {\fcolorbox{black}{black!15}{\color{blue}{\textbf{B2S17}}}}
\end{overpic}
\end{minipage}
\end{tabular} &  \begin{tabular}[c]{@{}l@{}}B2S13\end{tabular}  &   \begin{tabular}[c]{@{}l@{}}B2S11\end{tabular}  & \begin{tabular}[c]{@{}l@{}}B2S17\end{tabular}  & \begin{tabular}[c]{@{}l@{}}B2S17\end{tabular}  &  \begin{tabular}[c]{@{}l@{}}B2S20\end{tabular}  & \begin{tabular}[c]{@{}l@{}}B2S17\end{tabular} \\

\begin{tabular}[c]{@{}l@{}}
\begin{minipage}[b]{0.4\linewidth}
\begin{overpic}[width=\textwidth]{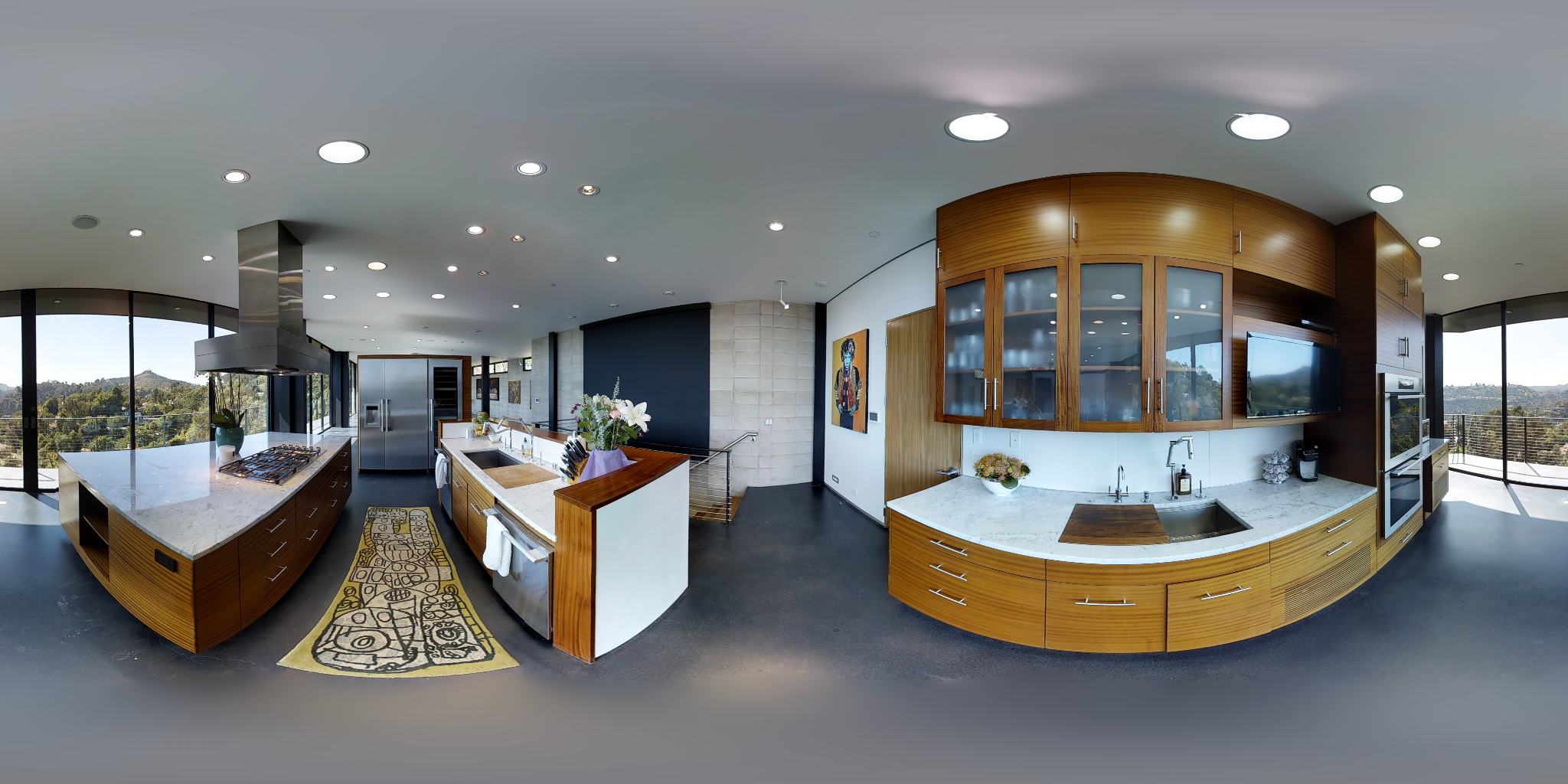}
\put (2,40) {\fcolorbox{black}{black!15}{\color{blue}{\textbf{B10S5}}}}
\end{overpic}
\end{minipage}
\end{tabular} & \begin{tabular}[c]{@{}l@{}}B10S13\end{tabular} &   \begin{tabular}[c]{@{}l@{}}B10S13\end{tabular}  &  \begin{tabular}[c]{@{}l@{}}B10S16\end{tabular}  &  \begin{tabular}[c]{@{}l@{}}B10S10\end{tabular}  &  \begin{tabular}[c]{@{}l@{}}B10S5\end{tabular}  & \begin{tabular}[c]{@{}l@{}}B10S5\end{tabular} \\

 \hline
\end{tabular}%
}
\vspace{-3mm}
\end{table}

Furthermore, we also analyze the robustness of the proposed localization method under varying focal lengths. We have considered a database that contains $360^{\circ}$ saliency graphs for each location, and query scenes are the images captured at different focal lengths. Table \ref{tab:articomp} reports results of the proposed $360^{\circ}$ saliency graph-based localization against a recent method, SOVS \cite{meena2024volumetric}, for $100$ scenes. The table indicates that accuracy decreases for both methods as focal length increases due to a reduction in scene-defining objects at higher magnification. However, the proposed $360^{\circ}$ saliency graph-based localization is much more robust compared to SOVS \cite{meena2024volumetric}. This indicates that incorporating a $360^{\circ}$ field of view improves localization robustness across varying focal lengths. 

 
\begin{table}[!h]
\vspace{-2.5mm}
    \centering
    \caption{Robustness of scene localization accuracy under varied focal length. }
    \label{tab:articomp}
    \vspace{-0.22cm}
\begin{tabular}{c|c|cc}
\hline
Method & $f$ &  $2f$   &  $3f$ \\  \hline

SOVS\cite{meena2024volumetric} & $0.47$  &  $0.36$ & $0.29$\\   
Proposed   & $0.65$  &  $0.50$ & $0.42$\\   
\hline
\end{tabular}
\vspace{-2.7mm}
\end{table}

\vspace{0.6ex}
\noindent \textbf{II. Navigation Performance:} 
To assess the effectiveness of $360^{\circ}$ saliency graph-based localization in `Vision-Language Navigation' (VLN), we perform a comparative evaluation of the proposed approach against several SOTA navigation methods. All methods are evaluated using the same topological map (i.e., a global map) but differ in localization and navigation strategies. The evaluation includes methods such as \cite{wang2019reinforced}, \cite{qi2020reverie}, \cite{gao2021room}, \cite{guhur2021airbert}, \cite{qiao2022hop}, \cite{chen2022think}, and \cite{qiao2023march} on \textit{REVERIE} \cite{qi2020reverie} dataset.
Figure \ref{fig:micpro} presents the navigation framework comprising two stages: saliency graph-based localization and navigation. In the localization stage, a saliency graph representation of the query scenes associated with each location is employed. This approach facilitates localization through a saliency graph matching. Subsequently, navigation is performed using a method in `March-in-Chat' (MiC) \cite{qiao2023march}.
The results are shown in Tab. \ref{tab:navcomp} for evaluation metrics \emph{Success Rate} (SR) \cite{qiao2023march}, \emph{Oracle Success Rate} (OSR) \cite{qiao2023march}, and \emph{Success rate weighted by trajectory Path Length} (SPL) \cite{qiao2023march}. 
The table shows that the `March-in-Chat' (MiC) \cite{qiao2023march} model outperforms SOTA methods for all metrics. The MiC \cite{qiao2023march} model includes a room-and-object aware scene perceiver (ROASP) for predicting room type and the objects visible at the current location using CLIP \cite{radford2021learning}. Replacing room-type prediction (i.e., localization) in ROASP with the proposed $360^{\circ}$ saliency graph-based localization method in MiC \cite{qiao2023march} as shown in Fig. \ref{fig:micpro} results in an increase in the value of metrics OSR, SR, and SPL by $4.41\%$, $2.81\%$, and $5.06\%$, respectively. This demonstrates that $360^{\circ}$ saliency graph-based localization benefits indoor navigation by improving the localization ability of agents.

\begin{figure}[!hbt]
\centering
\vspace{-2mm}
\begin{tikzpicture}
\node [input, name=input] {};         

\node [above right=-1.25cm and 0.5cm of input] (ip1) {\includegraphics[width=0.78\columnwidth, height=1.9in]{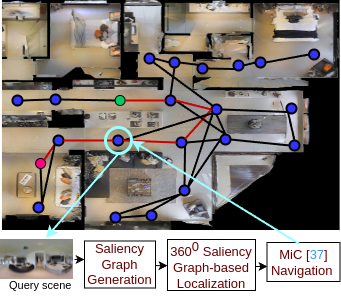}};

\end{tikzpicture}
\vspace{-0.5cm}
\caption{An example of proposed localization + MiC\cite{qiao2023march} navigation method.} 
\label{fig:micpro}
\vspace{-2mm}
\end{figure}

\begin{table}[!hbt]
\vspace{-2.5mm}
    \centering
    \caption{Impact of proposed localization method on navigation performance. }
    \label{tab:navcomp}
    \vspace{-0.24cm}
    \resizebox{\columnwidth}{!}{%
\begin{tabular}{c|ccc}
\hline
Methods & \multicolumn{3}{c}{Navigation Acc.} \\ \cline{2-4}
         & OSR($\uparrow$) & SR($\uparrow$)   &  SPL($\uparrow$) \\ \hline

FAST-MATTN\cite{qi2020reverie}  & $28.20$   & $14.40 $   & $7.19$ \\
CKR\cite{gao2021room}& $31.44$  & $19.14 $  & $11.84$ \\
Airbert\cite{guhur2021airbert}  & $34.51$   & $27.89$   & $21.88$ \\
HOP\cite{qiao2022hop}   & $36.24$   & $31.78$  & $26.11 $\\
DUET\cite{chen2022think}  & $51.07$  & $46.98$   & $33.73$ \\
MiC\cite{qiao2023march} & $\mathbf{62.37}$  &  $\mathbf{56.97}$ & $\mathbf{43.60}$\\
\begin{tabular}[c]{@{}l@{}}Proposed localization+MiC\cite{qiao2023march}\end{tabular}  & $\mathbf{66.78}$ & $\mathbf{59.78}$ &  $\mathbf{48.66}$ \\  
\hline
 \end{tabular}
}
\vspace{-3mm}
\end{table}

\begin{figure}[!hbt]
\vspace{-1.2mm}
\begin{tikzpicture}

\node [input, name=input] {};         


\node [above right=-1.25cm and -13.9cm of input] (im1) {\includegraphics[width=0.36\columnwidth]{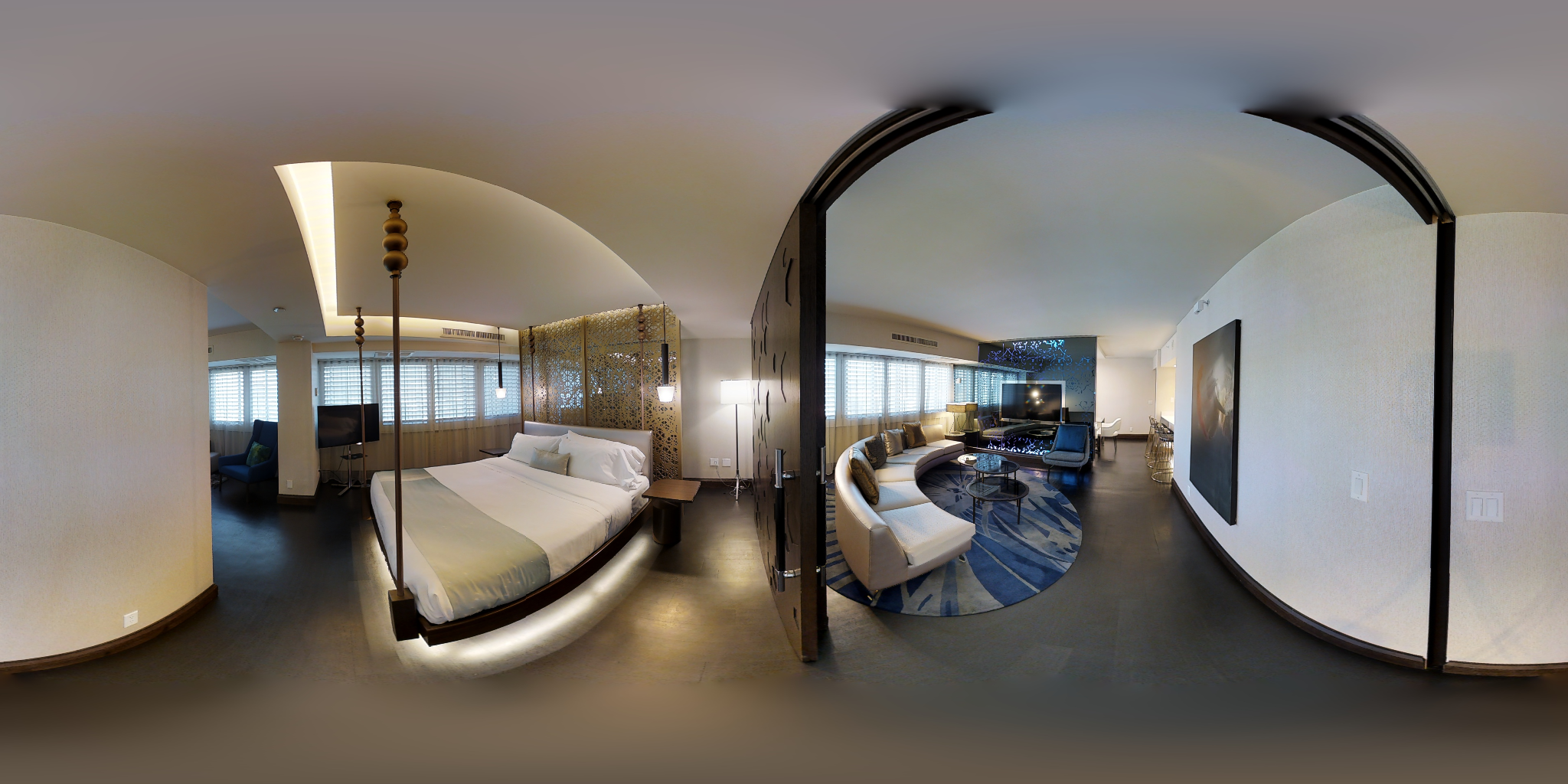} \includegraphics[width=0.36\columnwidth]{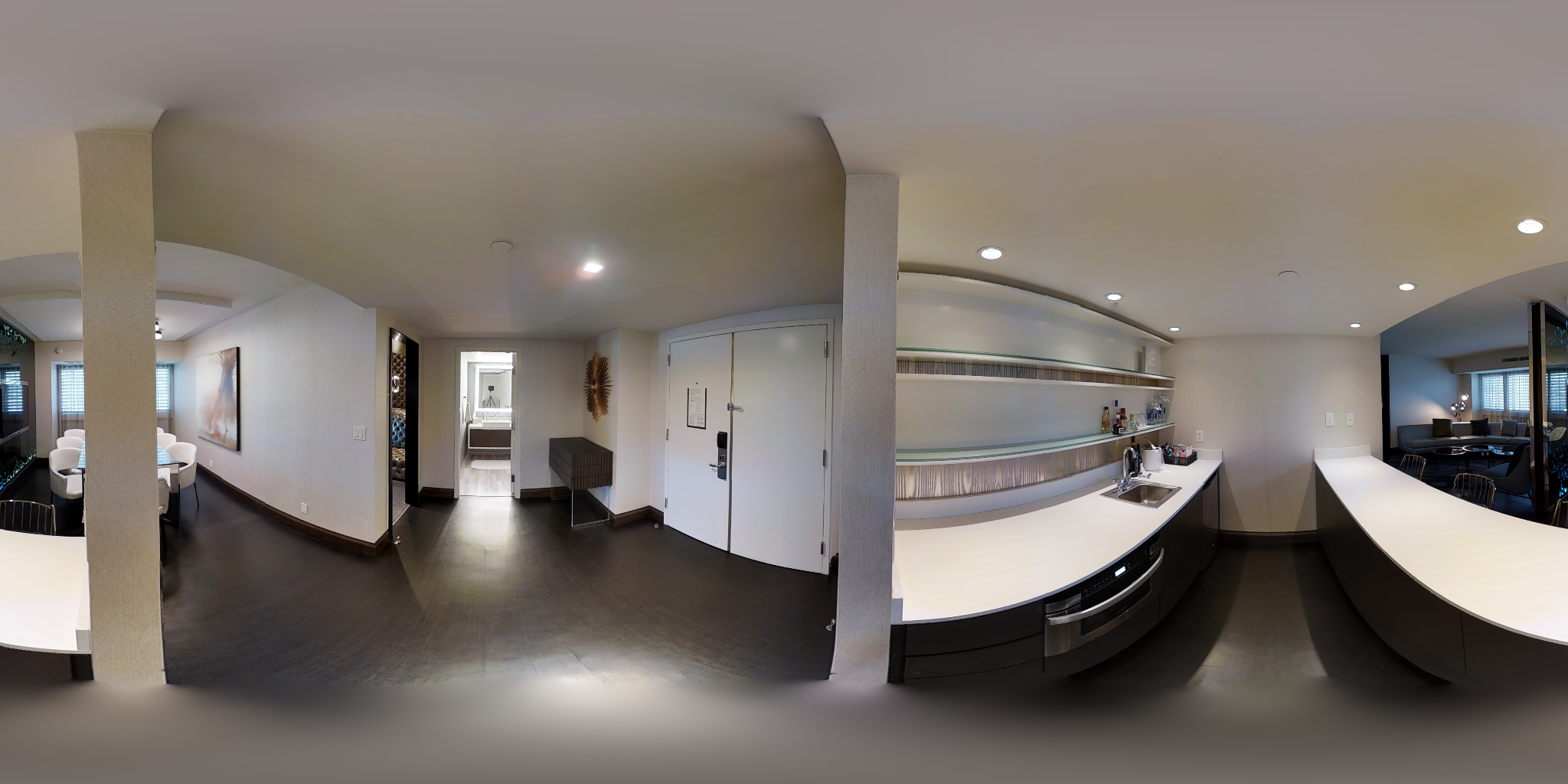} \quad \ };

\node[above right=-0.6cm and 0.26cm of im1.north west, font=\bfseries,blue] (lab1) {\textbf{bedroom \hspace{2.4cm} kitchen}};

\node [below of= im1,  node distance=1.1cm, align=center]  (rt) {\hspace{-1.7cm}\footnotesize{\textbf{room-type:} \hspace{0.2cm} bedroom \hspace{2.3cm} meetingroom} };

\node [below left=0.1cm and -2.3cm of rt, align=center]  (obt) {\footnotesize{\textbf{objects:} [`ceiling bedroom entry',}\\ \footnotesize{`floor room above', `ceiling inset}\\ \footnotesize{for fan']} };

\node [right of= obt,  node distance=4.35cm, align=center]  (obbtt) {\footnotesize{[`ceiling bedroom entry', `floor}\\ \footnotesize{room above', `ceiling inset}\\ \footnotesize{for fan']} };

\node [block, draw=black!50, inner sep=0pt, outer sep=0pt, fit = (rt) (obt) (obbtt)] (mic) {};

\node [below of= im1,  node distance=3.1cm, align=center]  (rt1) {\hspace{-2.1cm}\footnotesize{\textbf{room-type:} \hspace{0.2cm} bedroom \hspace{2.7cm} kitchen} };

\node [below left=0.1cm and -1.7cm of rt1, align=center]  (obt1) {\footnotesize{\textbf{objects:} [`table', `painting',}\\ \footnotesize{`bed', `tv', `door', `sofa',}\\ \footnotesize{`sofachair', `lamp', `tv', `door'] } };

\node [right of= obt1,  node distance=4.2cm, align=center]  (obt2) {\footnotesize{['kitchen counter', `kitchen counter', }\\ \footnotesize{`door', `table', `painting',}\\ \footnotesize{ `door', `chair', `chair'] } };

\node [block, draw=black!50, inner sep=0pt, outer sep=0pt, fit = (rt1) (obt1) (obt2) ] (pro) {};

\end{tikzpicture}
\vspace{-3.1mm}
\caption{An example showing the room-type prediction using MiC \cite{qiao2023march} and proposed localization method.  } 
\label{fig:navqua}
\vspace{-4.5mm}
\end{figure}

Figure \ref{fig:navqua} shows a visual comparison of the localization (room-type prediction) performed using the proposed and MiC \cite{qiao2023march} method for a few samples of $360^{\circ}$ scene. We observe that the MiC \cite{qiao2023march} is susceptible to failure for scenes composed of identical objects. For example, in MiC \cite{qiao2023march}, the bedroom and kitchen scenes are represented/characterized by similar objects, such as a ceiling inset for a fan, etc., leading to the inaccurate localization of the kitchen as a meeting room. In contrast, the proposed scene representation that employed a $360^{\circ}$ saliency graph effectively captures meaningful information about the scene based on salient objects present within each scene and provides a unique/distinctive representation for each scene. For example, the bedroom scene is represented by a bed, lamp, etc., whereas a kitchen counter characterizes the kitchen. Consequently, the proposed $360^{\circ}$ saliency graph-based localization results in improved indoor navigation. 


\begin{table}[!hbt]
    \centering
    \vspace{-2.5mm}
    \caption{Performance comparison of proposed navigation framework against MiC\cite{qiao2023march}. }
    \label{tab:navcompsota}
    \vspace{-0.24cm}
\begin{tabular}{c|ccc}
\hline
Methods & \multicolumn{3}{c}{Navigation Acc.} \\ \cline{2-3}
         & SR($\uparrow$)   &  SPL($\uparrow$) \\ \hline
MiC\cite{qiao2023march} &  ${56.97}$ & ${43.60}$\\ 
\begin{tabular}[c]{@{}l@{}}Proposed Navigation\end{tabular}  & $\mathbf{76.92}$ &  $\mathbf{72.85}$ \\ 
\hline
 \end{tabular}
\vspace{-3mm}
\end{table}

We also present a direct comparison between the proposed navigation framework and MiC \cite{qiao2023march} to assess overall performance. Table \ref{tab:navcompsota} shows the performance metrics for both methods. The results demonstrate that our approach, which leverages $360^{\circ}$ saliency graph-based localization and positioning, outperforms MiC \cite{qiao2023march}. This highlights the effectiveness of a proposed topological map, where each key location is represented by a $360^{\circ}$ saliency graph, providing a robust and informative map for the indoor environment.  

 
\vspace{-0.25cm}
\subsection{Ablation Study}
\vspace{-0.15cm}
We demonstrate the significance of components used in the proposed method by observing performance variation via the inclusion or exclusion of these components. 

\begin{table}[!hbt]
    \centering
    \vspace{-2.5mm}
    \caption{Impact of different components on localization. (USG: unweighted saliency graph)}
    \label{tab:mvcompabln}
    \vspace{-0.22cm}
    \renewcommand{\arraystretch}{0.95}
    \resizebox{\columnwidth}{!}{%
\begin{tabular}{c|ccccc}
\hline
Methods &  Image & Objects  & Salient objects & USG & $360^{\circ}$ saliency graph \\  \hline
\textit{Acc} ($\uparrow$)  & $0.68$ & $0.76$ &  $0.82$ & $0.84$ & $\mathbf{0.86}$\\
\hline
 \end{tabular}
}
\vspace{-3.5mm}
\end{table}
Table \ref{tab:mvcompabln} lists the \textit{Acc} obtained for localization using image features, objects, salient objects, unweighted saliency graph (USG), and $360^{\circ}$ saliency graph. For image feature extraction, we employ the ResNet-$101$ model pre-trained on ImageNet. Object-based localization is performed using the Jaccard similarity between the sets of detected object labels in the query and dataset images.
As indicated in the table, image feature-based localization achieves the lowest accuracy, as it neglects additional object attributes such as geometric features. Localization based on a collection of detected objects results in an improved performance compared to image feature-based localization. A further increase in performance metric is observed when using a collection of salient objects for localization, as this approach incorporates saliency scores to emphasize more informative objects. This highlights the significance of saliency in improving scene representation. However, while saliency helps prioritize relevant objects, it does not capture contextual relationships between them. Although an unweighted saliency graph yields further improvement by capturing structural context, it ignores edge significance. In contrast, the proposed method outperforms by using $360^{\circ}$ saliency graph that leverages salient objects, contextual information, edge weights through volumetric saliency scoring, and spherical orientation. This demonstrates the significance of the saliency and $360^{\circ}$ orientation used in this paper. 



\begin{table}[!hbt]
    \centering
    \vspace{-2.5mm}
    \caption{Impact of $360^0$ saliency graph-based positioning on navigation performance. }
    \label{tab:positnabl}
   \vspace{-0.22cm}
    \resizebox{0.94\columnwidth}{!}{%
\begin{tabular}{c|cc}
\hline
Method & without $360^0$ saliency graph & with $360^0$ saliency graph \\  \hline
SR & $58.73$ & $\mathbf{76.92}$\\
\hline
 \end{tabular}
}
\vspace{-3.5mm}
\end{table}

Furthermore, we assess the impact of $360^{\circ}$ saliency graph-based positioning on navigation performance in Tab. \ref{tab:positnabl}, using the SR as the evaluation metric. The table indicates that incorporating $360^{\circ}$ saliency graph-based positioning significantly increases the navigation performance by improving orientation estimation and localization. This improvement results from the $360^{\circ}$ saliency graph’s ability to effectively capture and integrate visual, contextual, semantic, and orientation information of salient objects, leading to a more informative and compact scene representation by removing redundant data (see supplementary for visual comparison).

 
\vspace{-0.27cm}
\subsection{Further Analysis}
\vspace{-0.15cm}
We also analyze the robustness of the proposed  $360^{\circ}$ saliency graph-based localization with respect to the camera parameter `Field Of View' (FOV).  A total of $82$, $360^{\circ}$ scenes are included for evaluation. Each $360^{\circ}$ scene is projected into a series of normal perspective scenes with FOV $60^0$, $120^0$, and $180^0$ using the toolbox in \cite{zhang2014panocontext}. The results for evaluation metrics are shown in Tab. \ref{tab:viewcomp}. The table shows that the lower FOV leads to poor performance metrics value due to challenges like identical scenes, non-scene-defining objects, and the absence of objects within the scene. In contrast, a larger FOV provides more detailed scene information, which leads to more accurate localization.

\begin{table}[h!]
\vspace{-2.5mm} 
\centering
\begin{minipage}{0.62\columnwidth}
\centering
\caption{Effect of FOV in localization performance.}
\label{tab:viewcomp}
\vspace{-0.25cm} 
\resizebox{0.99\columnwidth}{!}{%
\begin{tabular}{c|cccc}
\hline
FOV & $60^0$ & $120^0$ & $180^0$ & $360^{\circ}$  \\  \hline
\textit{Acc} & $0.37$ & $0.68$ & $0.79$ & $0.91$\\
Precision & $0.44$ & $0.68$ & $0.80$ & $0.94$ \\
Recall  & $0.63$ & $0.96$& $0.98$ & $0.96$ \\
\hline
 \end{tabular}
}
\end{minipage}%
\hspace{0.22cm} 
\begin{minipage}{0.34\columnwidth}
\centering
\caption{Localization on large-scale scenes.}
\label{tab: scale}
\vspace{-0.25cm} 
\resizebox{\columnwidth}{!}{%
\begin{tabular}{c|c}
\hline
\# scenes or rooms & \textit{Acc}\\ \hline
$6$ & $1.00$\\
$12$ & $1.00$\\
$42$ & $0.86$\\
$106$ & $0.82$ \\ 
\hline
\end{tabular}
}
\end{minipage}
\vspace{-3.6mm} 
\end{table}

We also assess the scalability of the proposed method to large-scale indoor scenes. Table \ref{tab: scale} presents the localization \textit{Accuracy} achieved using the proposed method for varying numbers of scenes per building. The proposed method achieved good \textit{Accuracy} for $106$ scenes, indicating that $360^{\circ}$ saliency graph-matching can scalably handle large buildings with multiple scenes/rooms.

\begin{table}[!hbt]
\vspace{-2.5mm}
    \centering
    \caption{Perturbation Analysis. }
    \label{tab:error}
     \vspace{-0.24cm}
    \resizebox{0.98\columnwidth}{!}{%
\begin{tabular}{l|p{1.6cm}|p{0.8cm}p{0.8cm}}
\hline
Method & Localiza-  & \multicolumn{2}{c}{Navigation Acc.} \\ \cline{3-4}
       &  tion Acc($\uparrow$)    &   SR($\uparrow$)   &  SPL($\uparrow$) \\ 
\hline
without perturbation &  $0.86$& $76.92$ & $72.85$\\
with spatial perturbation & $0.82$& $71.43$ & $63.68$\\
with orientation perturbation &  $0.74$& $68.75$ & $62.05$\\
with object perturbation ($50$\%) $P_1$  & $0.78$  & $69.71$ & $61.21$\\ 
with object perturbation ($66$\%) $P_2$  & $0.67$  & $40.28$ & $31.77$\\
\hline
 \end{tabular}
 }
\vspace{-2.8mm}
\end{table}

We further analyzed the effect of various perturbations, including spatial, orientation, and object variations, on the robustness of the proposed method.
Table \ref{tab:error} lists the localization and navigation performance metrics using the proposed method under different types of perturbation. We notice from the table that the proposed method is robust to small perturbations. The large amount of object perturbation results in the lack of correct scene information, leading to limited performance. 
 
\vspace{-0.18cm}
\section{Conclusion}
\label{sec:conclusion}
\vspace{-0.15cm}
This paper introduces a scene representation framework for indoor environments based on a topological map, where each key location is accompanied by an orientation-rich $360^{\circ}$ saliency graph generated from a panorama RGB-D image to facilitate navigation. The proposed $360^{\circ}$ saliency graph encodes rich visual, contextual, semantic, and geometric information by representing salient objects as nodes and their relations as edges. In addition, this paper presents a graph-matching–based localization and positioning method to demonstrate the applicability of the $360^{\circ}$ saliency graph representation in vision-based indoor navigation. Experimental results highlight the efficacy of the proposed approach by significantly improving localization performance and robustness compared to existing methods. Consequently, the enhanced localization yields an improved navigation performance. However, the method exhibits limited capability in dynamic environments, where rapid scene changes can degrade localization reliability and, subsequently, navigation performance.
In the future, we would like to address these limitations.

 

{
\bibliographystyle{IEEEtran}
\vspace{-0.25cm}
    \bibliography{main}

@inproceedings{koh2021pathdreamer,
  title={Pathdreamer: A world model for indoor navigation},
  author={Koh, Jing Yu and Lee, Honglak and Yang, Yinfei and Baldridge, Jason and Anderson, Peter},
  booktitle={ICCV},
  year={2021}
}

@inproceedings{qiao2023march,
  title={March in chat: Interactive prompting for remote embodied referring expression},
  author={Qiao, Yanyuan and Qi, Yuankai and Yu, Zheng and Liu, Jing and Wu, Qi},
  booktitle={ICCV},
  year={2023}
}

@inproceedings{wang2019reinforced,
  title={Reinforced cross-modal matching and self-supervised imitation learning for vision-language navigation},
  author={Wang, Xin and Huang, Qiuyuan and Celikyilmaz, Asli and Gao, Jianfeng and Shen, Dinghan and Wang, Yuan-Fang and Wang, William Yang and Zhang, Lei},
  booktitle={CVPR},
  year={2019}
}

@inproceedings{qiao2022hop,
  title={Hop: History-and-order aware pre-training for vision-and-language navigation},
  author={Qiao, Yanyuan and Qi, Yuankai and Hong, Yicong and Yu, Zheng and Wang, Peng and Wu, Qi},
  booktitle={CVPR},
  year={2022}
}

@inproceedings{guhur2021airbert,
  title={Airbert: In-domain pretraining for vision-and-language navigation},
  author={Guhur, Pierre-Louis and Tapaswi, Makarand and Chen, Shizhe and Laptev, Ivan and Schmid, Cordelia},
  booktitle={ICCV},
  year={2021}
}

@inproceedings{chen2022think,
  title={Think global, act local: Dual-scale graph transformer for vision-and-language navigation},
  author={Chen, Shizhe and Guhur, Pierre-Louis and Tapaswi, Makarand and Schmid, Cordelia and Laptev, Ivan},
  booktitle={CVPR},
  year={2022}
}

@inproceedings{radford2021learning,
  title={Learning transferable visual models from natural language supervision},
  author={Radford, Alec and Kim, Jong Wook and Hallacy, Chris and Ramesh, Aditya and Goh, Gabriel and Agarwal, Sandhini and Sastry, Girish and Askell, Amanda and Mishkin, Pamela and Clark, Jack and others},
  booktitle={ICML},
  year={2021},
  organization={PMLR}
}

@inproceedings{qi2020reverie,
  title={Reverie: Remote embodied visual referring expression in real indoor environments},
  author={Qi, Yuankai and Wu, Qi and Anderson, Peter and Wang, Xin and Wang, William Yang and Shen, Chunhua and Hengel, Anton van den},
  booktitle={CVPR},
  year={2020}
}

@article{xu2023free,
  title={Free-Viewpoint Navigation of Indoor Scene with 360° Field of View},
  author={Xu, Hang and Zhao, Qiang and Ma, Yike and Wang, Shuai and Yan, Chenggang and Dai, Feng},
  journal={Electronics},
  year={2023},
  publisher={MDPI}
}

@misc{kim2022habitatmap,
  author = {Nuri Kim},
  title = {Object2HabitatMap},
  howpublished = {\url{https://github.com/bareblackfoot/Object2HabitatMap}},
  year = {2022}
}

@inproceedings{kwon2024wayil,
  title={WayIL: Image-based Indoor Localization with Wayfinding Maps},
  author={Kwon, Obin and Jung, Dongki and Kim, Youngji and Ryu, Soohyun and Yeon, Suyong and Oh, Songhwai and Lee, Donghwan},
  booktitle={ICRA},
  year={2024},
  organization={IEEE}
}

@inproceedings{raj2016appearance,
  title={Appearance-based indoor navigation by IBVS using mutual information},
  author={Raj, Suman and Giordano, Paolo Robuffo and Chaumette, Fran{\c{c}}ois},
  booktitle={ICARCV},
  year={2016},
  organization={IEEE}
}

@article{an2024etpnav,
  title={Etpnav: Evolving topological planning for vision-language navigation in continuous environments},
  author={An, Dong and Wang, Hanqing and Wang, Wenguan and Wang, Zun and Huang, Yan and He, Keji and Wang, Liang},
  journal={PAMI},
  year={2024},
  publisher={IEEE}
}

@inproceedings{wang2024graph,
  title={Graph-Based Environment Representation for Vision-and-Language Navigation in Continuous Environments},
  author={Wang, Ting and Wu, Zongkai and Yao, Feiyu and Wang, Donglin},
  booktitle={ICASSP},
  year={2024},
  organization={IEEE}
}

@article{mehan2024questmaps,
  title={QueSTMaps: Queryable Semantic Topological Maps for 3D Scene Understanding},
  author={Mehan, Yash and Gupta, Kumaraditya and Jayanti, Rohit and Govil, Anirudh and Garg, Sourav and Krishna, Madhava},
  journal={arXiv preprint arXiv:2404.06442},
  year={2024}
}

@inproceedings{guerrero2020s,
  title={What’s in my room? object recognition on indoor panoramic images},
  author={Guerrero-Viu, Julia and Fernandez-Labrador, Clara and Demonceaux, C{\'e}dric and Guerrero, Jose J},
  booktitle={ICRA},
  year={2020},
  organization={IEEE}
}

@article{meena2024volumetric,
  title={A Volumetric Saliency Guided Image Summarization for RGB-D Indoor Scene Classification},
  author={Meena, Preeti and Kumar, Himanshu and Yadav, Sandeep},
  journal={TCSVT},
  year={2024},
  publisher={IEEE}
}

@inproceedings{meena2024indoor,
  title={An Indoor Scene Localization Method Using Graphical Summary of Multi-View RGB-D Images},
  author={Meena, Preeti and Kumar, Himanshu and Yadav, Sandeep},
  booktitle={ICIP},
  pages={3306--3312},
  year={2024},
  organization={IEEE}
}

@book{gonzalez2009digital,
  title={Digital image processing},
  author={Gonzalez, Rafael C},
  year={2009},
  publisher={Pearson education india}
}

@article{lee2021visibility,
  title={Visibility graph-based path-planning algorithm with quadtree representation},
  author={Lee, Wonhee and Choi, Gwang-Hyeok and Kim, Tae-wan},
  journal={Applied Ocean Research},
  volume={117},
  pages={102887},
  year={2021},
  publisher={Elsevier}
}

@article{oliva2020ganalyze,
  title={Ganalyze: Toward visual definitions of cognitive image properties},
  author={Oliva, Aude and Isola, Phillip and others},
  journal={Journal of Vision},
  volume={20},
  number={11},
  pages={297--297},
  year={2020},
  publisher={The Association for Research in Vision and Ophthalmology}
}

@article{qin2021semantic,
  title={Semantic loop closure detection based on graph matching in multi-objects scenes},
  author={Qin, Cao and Zhang, Yunzhou and Liu, Yingda and Lv, Guanghao},
  journal={Journal of Visual Communication and Image Representation},
  volume={76},
  pages={103072},
  year={2021},
  publisher={Elsevier}
}

@inproceedings{pramatarov2022boxgraph,
  title={BoxGraph: Semantic place recognition and pose estimation from 3D LiDAR},
  author={Pramatarov, Georgi and De Martini, Daniele and Gadd, Matthew and Newman, Paul},
  booktitle={IROS},
  pages={7004--7011},
  year={2022},
  organization={IEEE}
}

@inproceedings{zhang2014panocontext,
  title={Panocontext: A whole-room 3d context model for panoramic scene understanding},
  author={Zhang, Yinda and Song, Shuran and Tan, Ping and Xiao, Jianxiong},
  booktitle={ECCV},
  year={2014},
  organization={Springer}
}

@inproceedings{grover2016node2vec,
  title={node2vec: Scalable feature learning for networks},
  author={Grover, Aditya and Leskovec, Jure},
  booktitle={ACM SIGKDD-KDD},
  year={2016}
}

@article{chang2017matterport3d,
  title={Matterport3d: Learning from rgb-d data in indoor environments},
  author={Chang, Angel and Dai, Angela and Funkhouser, Thomas and Halber, Maciej and Niessner, Matthias and Savva, Manolis and Song, Shuran and Zeng, Andy and Zhang, Yinda},
  journal={arXiv preprint arXiv:1709.06158},
  year={2017}
}

@article{chen2019behavioral,
  title={A behavioral approach to visual navigation with graph localization networks},
  author={Chen, Kevin and De Vicente, Juan Pablo and Sepulveda, Gabriel and Xia, Fei and Soto, Alvaro and V{\'a}zquez, Marynel and Savarese, Silvio},
  journal={arXiv preprint arXiv:1903.00445},
  year={2019}
}

@inproceedings{chaplot2020neural,
  title={Neural topological slam for visual navigation},
  author={Chaplot, Devendra Singh and Salakhutdinov, Ruslan and Gupta, Abhinav and Gupta, Saurabh},
  booktitle={CVPR},
  year={2020}
}

@inproceedings{du2020learning,
  title={Learning object relation graph and tentative policy for visual navigation},
  author={Du, Heming and Yu, Xin and Zheng, Liang},
  booktitle={ECCV},
  year={2020},
  organization={Springer}
}

@inproceedings{gao2021room,
  title={Room-and-object aware knowledge reasoning for remote embodied referring expression},
  author={Gao, Chen and Chen, Jinyu and Liu, Si and Wang, Luting and Zhang, Qiong and Wu, Qi},
  booktitle={CVPR},
  year={2021}
}

@inproceedings{xu2024robot,
  title={Robot Navigation in Unseen Environments using Coarse Maps},
  author={Xu, Chengguang and Amato, Christopher and Wong, Lawson LS},
  booktitle={ICRA},
  year={2024},
  organization={IEEE}
}

@article{zhou2022relational,
  title={Relational attention-based Markov logic network for visual navigation},
  author={Zhou, Kang and Guo, Chi and Zhang, Huyin},
  journal={The Journal of Supercomputing},
  volume={78},
  number={7},
  year={2022},
  publisher={Springer}
}

@article{niu2019hand,
  title={A hand-drawn map-based navigation method for mobile robots using objectness measure},
  author={Niu, Jie and Qian, Kun},
  journal={IJARS},
  volume={16},
  number={3},
  year={2019},
  publisher={SAGE Publications Sage UK: London, England}
}

@inproceedings{moravec1985high,
  title={High resolution maps from wide angle sonar},
  author={Moravec, Hans and Elfes, Alberto},
  booktitle={ICRA},
  volume={2},
  year={1985},
  organization={IEEE}
}

@inproceedings{konolige2011navigation,
  title={Navigation in hybrid metric-topological maps},
  author={Konolige, Kurt and Marder-Eppstein, Eitan and Marthi, Bhaskara},
  booktitle={ICRA},
  year={2011},
  organization={IEEE}
}

@inproceedings{detone2018superpoint,
  title={Superpoint: Self-supervised interest point detection and description},
  author={DeTone, Daniel and Malisiewicz, Tomasz and Rabinovich, Andrew},
  booktitle={CVPR},
  year={2018}
}

@article{zhao2022alike,
  title={Alike: Accurate and lightweight keypoint detection and descriptor extraction},
  author={Zhao, Xiaoming and Wu, Xingming and Miao, Jinyu and Chen, Weihai and Chen, Peter CY and Li, Zhengguo},
  journal={TMM},
  volume={25},
  year={2022},
  publisher={IEEE}
}

@inproceedings{luo2020aslfeat,
  title={Aslfeat: Learning local features of accurate shape and localization},
  author={Luo, Zixin and Zhou, Lei and Bai, Xuyang and Chen, Hongkai and Zhang, Jiahui and Yao, Yao and Li, Shiwei and Fang, Tian and Quan, Long},
  booktitle={CVPR},
  year={2020}
}

@inproceedings{jiang2021cotr,
  title={Cotr: Correspondence transformer for matching across images},
  author={Jiang, Wei and Trulls, Eduard and Hosang, Jan and Tagliasacchi, Andrea and Yi, Kwang Moo},
  booktitle={ICCV},
  year={2021}
}

@inproceedings{sarlin2020superglue,
  title={Superglue: Learning feature matching with graph neural networks},
  author={Sarlin, Paul-Edouard and DeTone, Daniel and Malisiewicz, Tomasz and Rabinovich, Andrew},
  booktitle={CVPR},
  year={2020}
}

@inproceedings{chen2022aspanformer,
  title={Aspanformer: Detector-free image matching with adaptive span transformer},
  author={Chen, Hongkai and Luo, Zixin and Zhou, Lei and Tian, Yurun and Zhen, Mingmin and Fang, Tian and Mckinnon, David and Tsin, Yanghai and Quan, Long},
  booktitle={ECCV},
  year={2022},
  organization={Springer}
}

@inproceedings{sun2021loftr,
  title={LoFTR: Detector-free local feature matching with transformers},
  author={Sun, Jiaming and Shen, Zehong and Wang, Yuang and Bao, Hujun and Zhou, Xiaowei},
  booktitle={CVPR},
  year={2021}
}

@article{tang2022quadtree,
  title={Quadtree attention for vision transformers},
  author={Tang, Shitao and Zhang, Jiahui and Zhu, Siyu and Tan, Ping},
  journal={arXiv preprint arXiv:2201.02767},
  year={2022}
}

@article{taioli2024mind,
  title={Mind the error! detection and localization of instruction errors in vision-and-language navigation},
  author={Taioli, Francesco and Rosa, Stefano and Castellini, Alberto and Natale, Lorenzo and Del Bue, Alessio and Farinelli, Alessandro and Cristani, Marco and Wang, Yiming},
  journal={arXiv preprint arXiv:2403.10700},
  year={2024}
}

@inproceedings{kim2023topological,
  title={Topological semantic graph memory for image-goal navigation},
  author={Kim, Nuri and Kwon, Obin and Yoo, Hwiyeon and Choi, Yunho and Park, Jeongho and Oh, Songhwai},
  booktitle={CoRL},
  year={2023},
  organization={PMLR}
}

@inproceedings{liu2024volumetric,
  title={Volumetric Environment Representation for Vision-Language Navigation},
  author={Liu, Rui and Wang, Wenguan and Yang, Yi},
  booktitle={CVPR},
  year={2024}
}

@article{zhang2023searching,
  title={Searching from area to point: A hierarchical framework for semantic-geometric combined feature matching},
  author={Zhang, Yesheng and Zhao, Xu and Qian, Dahong},
  journal={arXiv preprint arXiv:2305.00194},
  year={2023}
}

@article{triana2018implementation,
  title={Implementation floyd-warshall algorithm for the shortest path of garage},
  author={Triana, Yaya Sudarya and Syahputri, Indah},
  journal={IJISRT},
  volume={3},
  number={2},
  year={2018}
}

@article{labrosse2007short,
  title={Short and long-range visual navigation using warped panoramic images},
  author={Labrosse, Fr{\'e}d{\'e}ric},
  journal={Robotics and Autonomous Systems},
  year={2007},
  publisher={Elsevier}
}

@article{yu2011image,
  title={Image-based homing navigation with landmark arrangement matching},
  author={Yu, Seung-Eun and Kim, DaeEun},
  journal={Information sciences},
  volume={181},
  number={16},
  year={2011},
  publisher={Elsevier}
}

@inproceedings{cha2012omni,
  title={Omni-directional image matching for homing navigation based on optical flow algorithm},
  author={Cha, Youngseo and Kim, DaeEun},
  booktitle={ICCAS},
  year={2012},
  organization={IEEE}
}

@article{lee2018visual,
  title={Visual homing navigation with Haar-like features in the snapshot},
  author={Lee, Changmin and Kim, Daeeun},
  journal={IEEE Access},
  volume={6},
  year={2018},
  publisher={IEEE}
}

@inproceedings{kim2024fully,
  title={Fully Geometric Panoramic Localization},
  author={Kim, Junho and Jeong, Jiwon and Kim, Young Min},
  booktitle={CVPR},
  year={2024}
}

@inproceedings{kim2023ldl,
  title={Ldl: Line distance functions for panoramic localization},
  author={Kim, Junho and Choi, Changwoon and Jang, Hojun and Kim, Young Min},
  booktitle={ICCV},
  year={2023}
}

@inproceedings{liu2024multiple,
  title={Multiple Visual Features in Topological Map for Vision-and-Language Navigation},
  author={Liu, Ruonan and Kong, Ping and Zhang, Weidong},
  booktitle={IROS},
  year={2024},
  organization={IEEE}
}

@inproceedings{krantz2021waypoint,
  title={Waypoint models for instruction-guided navigation in continuous environments},
  author={Krantz, Jacob and Gokaslan, Aaron and Batra, Dhruv and Lee, Stefan and Maksymets, Oleksandr},
  booktitle={ICCV},
  pages={15162--15171},
  year={2021}
}

@inproceedings{he2016deep,
  title={Deep residual learning for image recognition},
  author={He, Kaiming and Zhang, Xiangyu and Ren, Shaoqing and Sun, Jian},
  booktitle={CVPR},
  pages={770--778},
  year={2016}
}

@inproceedings{hong2023learning,
  title={Learning navigational visual representations with semantic map supervision},
  author={Hong, Yicong and Zhou, Yang and Zhang, Ruiyi and Dernoncourt, Franck and Bui, Trung and Gould, Stephen and Tan, Hao},
  booktitle={ICCV},
  pages={3055--3067},
  year={2023}
}

@inproceedings{li2022envedit,
  title={Envedit: Environment editing for vision-and-language navigation},
  author={Li, Jialu and Tan, Hao and Bansal, Mohit},
  booktitle={CVPR},
  pages={15407--15417},
  year={2022}
}

@inproceedings{khandelwal2022simple,
  title={Simple but effective: Clip embeddings for embodied ai},
  author={Khandelwal, Apoorv and Weihs, Luca and Mottaghi, Roozbeh and Kembhavi, Aniruddha},
  booktitle={CVPR},
  year={2022}
}

@inproceedings{du2021curious,
  title={Curious representation learning for embodied intelligence. IEEE},
  author={Du, Yilun and Gan, Chuang and Isola, Phillip},
  booktitle={ICCV},
  pages={10388--10397},
  year={2021}
}

@inproceedings{xie2024mfp,
  title={MFP-CNN: Multi-Scale Fusion and Pooling Network for Accurate Scene Classification},
  author={Xie, Wanjing and Liu, Tianjie},
  booktitle={ICCVIT},
  year={2024},
  organization={IEEE}
}
}



\end{document}